\documentclass[10pt,twocolumn,letterpaper]{article}

\usepackage{cvpr}
\usepackage{times}
\usepackage{epsfig}
\usepackage{graphicx}
\usepackage{amsmath}
\usepackage{amssymb}

\cvprfinalcopy
\def\cvprPaperID{4802} 

\ifcvprfinal\fi

\usepackage{adjustbox}
\usepackage[utf8]{inputenc} 
\usepackage[T1]{fontenc}    
\usepackage{hyperref}       
\usepackage{url}            
\usepackage{booktabs}       
\usepackage{amsfonts}       
\usepackage{nicefrac}       
\usepackage{microtype}      

\usepackage{longtable}
\usepackage{graphicx}
\usepackage{booktabs} 

\usepackage{times}
\usepackage{epsfig}
\usepackage{amsmath}
\usepackage{amssymb}
\usepackage{amsthm}
\usepackage{thmtools}
\usepackage{microtype}
\usepackage{fixmath}
\usepackage{bbm}
\usepackage{enumitem}

\usepackage{amsxtra}
\usepackage{amsfonts}
\usepackage{color}
\usepackage{mathtools}
\usepackage{algorithm,algorithmic} 
\usepackage{capt-of}
\usepackage[bottom]{footmisc}

\newtheorem{lemma}{Lemma}
\newtheorem*{lemma*}{Lemma}
\newtheorem{proposition}{Proposition}
\newtheorem*{proposition*}{Proposition}
\newtheorem*{theorem*}{Theorem}
\newtheorem{remark}{Remark}

\DeclareMathOperator*{\dom}{dom} 
\DeclareMathOperator*{\argmin}{arg\,min} 
\DeclareMathOperator*{\argmax}{arg\,max} 
\def\conv{\operatorname{conv}}
\def\map{\operatorname{map}}

\def\CX{{\mathbb X}}
\providecommand{\abs}[1]{\lvert#1\rvert}
\providecommand{\norm}[1]{\lVert#1\rVert}
\newcommand{\R}{\mathbb{R}}
\newcommand{\N}{\mathbb{N}}
\newcommand{\eins}{\mathbbmss{1}}

\newcommand{\SX}{\mathcal{X}}
\newcommand{\SY}{\mathcal{Y}}
\newcommand{\SL}{\mathcal{L}}

\newcommand{\Y}{\mathbb{Y}}

\def\myparagraph#1{\vspace{-1pt}\noindent{\bf #1~~}}

\title{MAP inference via Block-Coordinate Frank-Wolfe Algorithm}

\author{Paul Swoboda\thanks{The work was performed while the first author was at IST Austria. The work was supported by the European Research Council under the European Unions Seventh Framework Programme (FP7/2007-2013)/ERC grant agreement no 616160.}\\
MPI for Informatics, Germany\\
{\tt\small pswoboda@mpi-inf.mpg.de}
\and
Vladimir Kolmogorov\\
IST Austria\\
{\tt\small vnk@ist.ac.at}
}

\begin{document}

\maketitle

\begin{abstract}
We present a new proximal bundle method for Maximum-A-Posteriori (MAP) inference in structured energy minimization problems.
The method optimizes a Lagrangean relaxation of the original energy minimization problem using a multi plane block-coordinate Frank-Wolfe method that takes advantage of the specific structure of the Lagrangean decomposition.
We show empirically that our method outperforms state-of-the-art Lagrangean decomposition based algorithms on some challenging Markov Random Field, multi-label discrete tomography and graph matching problems.
\end{abstract}

\section{Introduction}
Maximum-A-Posteriori (MAP) inference, that is
minimizing an energy function $f: \SX \rightarrow \R$ over a discrete set of labelings $\SX$
is a central tool in computer vision and machine learning.
Many solvers have been proposed for various special forms of energy $f$ and labeling space $\SX$, see~\cite{OpenGMBenchmark} for an overview of solvers and applications for the prominent special case of Markov Random Fields (MRF).
Solvers can roughly be categorized into three categories:
(i) \emph{Exact solvers} that use search techniques (e.g. branch-and-bound) and possibly rely on solving lower bounds with LP-solvers to speed up search,
(ii) \emph{primal heuristics} that find a suboptimal solution with an ad-hoc search strategy adapted to the specific problem and
(iii) \emph{Lagrangean decomposition} (a.k.a.\ dual decomposition) based algorithms that decompose the original problem into smaller efficiently optimizable subproblems and exchange Lagrangean multipliers between subproblems until consensus between subproblems is achieved.

Except when the energy fulfills special assumptions, exact solvers are usually not applicable, since problem sizes in computer vision are too large.
On the other hand, primal heuristics can be fast but solution quality need not be good and no information is given to judge it.
Moreover, algorithms from the first two paradigms are usually developed ad-hoc for specific optimization problems and cannot be easily extended to other ones.
Lagrangean decomposition based algorithms are a good middle ground, since they optimize a dual lower bound, hence can output a gap that shows the distance to optimum, yet use techniques that scale to large problem sizes.
Generalization to new problems is also usually much easier, since subproblems can be readily combined.

A large number of algorithmic techniques have been proposed for optimizing a Lagrangean decomposition for MRFs, including
(i) \emph{Message passing}~\cite{SRMPKolmogorov,Werner07,IRPSLP,MPLP} (a.k.a.\ block coordinate ascent, belief propagation),
(ii) \emph{first order proximal splitting methods}~\cite{RavikumarProximalMethodsMAPMRF,schmidt2011evaluationProximalMAP} and
(iii) \emph{Smoothing based methods}~\cite{LagrangeanRelaxationJohnsonMalioutov,AdaptiveDiminishingSmoothingSavchynskyyUAI},
(iv) \emph{Nesterov schemes}~\cite{savchynskyy2011study,AcceleratedMAPJojic},
(v) \emph{mirror descent}~\cite{MapMirrorDescent},
(vi) \emph{subgradient based algorithms}~\cite{SchlesingerSubgradient,StorvikDahlLagrangeanBasedMAP,DualDecompositionKomodakis}.
In the case of MAP inference in MRFs, the study~\cite{OpenGMBenchmark} has shown that message passing techniques outperform competing Lagrangean decomposition based methods by a large margin.
However, there are two main practical shortcomings of message passing algorithms:
(i) they need not converge to the optimum of the relaxation corresponding to the Lagrangean decomposition: while well-designed algorithms monotonically improve a dual lower bound, they may get stuck in suboptimal fixed points. 
(ii) So called min-marginals must be computable fast for all the subproblems in the given Lagrangean decomposition.
While the above problems do not seem to be an issue for most MAP inference tasks in MRFs, for other problems they are.
In such cases, alternative techniques must be used. 
Subgradient based methods can help here, since they do not possess the above shortcomings: They converge to the optimum of the Lagrangean relaxation and only require finding solutions to the subproblems of the decomposition, which is easier than their min-marginals (as needed for (i)), proximal steps (as needed for (ii)) or smoothed solutions (as needed for (iii) and (iv)).

The simplest subgradient based algorithm is subgradient ascent. However, its convergence is typically slow.
Bundle methods, which store a series of subgradients to build a local approximation of the function to be optimized, empirically converge faster.

\subsection{Contribution \& Organization}
We propose a multi plane block-coordinate version of the Frank-Wolfe method to find minimizing directions in a proximal bundle framework, see Section~\ref{sec:method}.
Our method exploits the structure of the problem's Lagrangean decomposition and is inspired by~\cite{MP-BCFW}.
Applications of our approach to MRFs, discrete tomography and graph matching are presented in Section~\ref{sec:applications}.
An experimental evaluation on these problems is given in Section~\ref{sec:experiments} and suggests that our method is superior to comparable established methods.

    All proofs are given in the Appendix~\ref{sec:appendix}.
A C++-implementation of our Frank-Wolfe method is available at \url{http://pub.ist.ac.at/~vnk/papers/FWMAP.html}.
The MRF, discrete tomography and graph matching solvers built on top of our method can be obtained at \url{https://github.com/LPMP/LPMP}.

\subsection{Related work}
To our knowledge, the Frank-Wolfe method has not yet been used in our general setting, i.e.\ underlying a proximal bundle solver for a general class of structured energy minimization problems.
Hence, we subdivide related work into (i) subgradient/bundle methods for energy minimization, (ii) ad-hoc approaches that use Frank-Wolfe for specific tasks and (iii) proximal bundle methods.

\myparagraph{Subgradient based solvers} have first been proposed by~\cite{StorvikDahlLagrangeanBasedMAP} for MAP-inference in MRFs and were later popularized by~\cite{DualDecompositionKomodakis}.
These works rely on a decomposition of MRFs into trees.
New decompositions for certain MRFs were introduced and optimized in~\cite{OsokinSubmodularRelaxationMRF} for max-flow subproblems and in~\cite{SchraudolphReweightedPerfectMatching} for perfect matching subproblems.  
The work~\cite{yarkony2010covering} used a covering of the graph by a tree and optimized additional equality constraints on duplicated nodes via subgradient ascent.
Usage of bundle methods that store a series of subgradients to build a local approximation of the objective function was proposed by~\cite{kappes2012bundle,SchraudolphReweightedPerfectMatching} for MAP inference in MRFs.

\myparagraph{The Frank-Wolfe} algorithm was developed in the 50s~\cite{FWolfe56} and was popularized recently by~\cite{jaggi2013revisiting}.
In~\cite{BCFW} a block coordinate version of Frank-Wolfe was proposed and applied to training structural SVMs.
Further improvements were given in~\cite{MP-BCFW,Osokin:ICML16}
where, among other things, caching of planes was proposed for the Frank-Wolfe method. Several works have applied Frank-Wolfe to the MAP-MRF inference problem: 
(1) \cite{Schwing:ICML14} used Frank-Wolfe to compute an approximated steepest-descent direction in the local polytope relaxation for MRFs. 
(2) \cite{Meshi:NIPS15} used Frank-Wolfe to solve a modified problem obtained by adding a strictly convex quadratic function to the original objective (either primal
or dual). In contrast to these works, we use Frank-Wolfe inside a proximal method. Furthermore, the papers above
use a fine decomposition into many small subproblems (corresponding to pairwise terms of the energy function), while
we decompose the problem into much larger subproblems.
In general, our decomposition results in fewer dual variables to optimize over, and each Frank-Wolfe update can be expected to give a much larger gain.
Frank-Wolfe was also used in \cite{denseCRF-CVPR17} for MAP inference in dense MRFs with Potts interactions and Gaussian weights.
As we do, they use Frank-Wolfe to optimize proximal steps for MAP-inference. In constrast to our work, they do not optimize Lagrangean multipliers, but the original variables directly.
In other words, they work in the primal, while we work in the dual.
We remark that our formulation is applicable to more general integer optimization problems than either~\cite{Schwing:ICML14,Meshi:NIPS15,denseCRF-CVPR17} and it does not seem straightforward to apply these approaches to our more general setting while only requiring access to MAP-oracles of subproblems.

\myparagraph{Proximal bundle methods} were introduced in~\cite{Kiwiel1985Bundle,Lemarechal1986Bundle} to accelerate subgradient descent algorithms. 
They work by locally building an approximation (bundle) to the function to be optimized and use this bundle to find a descent direction.
For stability, a quadratic (proximal) term is added~\cite{ProximityControlBundleMethodsKiwiel}.
While not theoretically guaranteed, proximal bundle methods are often faster than subgradient methods.

\section{Method}
\label{sec:method}

\myparagraph{Original problem} 
We consider the problem of minimizing a function of Boolean variables
represented as a sum of individual terms:
\begin{equation}
  \label{eq:decomposition}
  \min_{x \in \{0,1\}^d} f(x), \quad f(x) := \sum_{t\in T} f_t(x_{A_t})
\end{equation}
Here term $t\in T$ is specified by a subset of variables $A_t\subseteq[d]$
and a function $f_t:\{0,1\}^{A_t}\rightarrow\mathbb R\cup\{+\infty\}$ of $|A_t|$ variables.
Vector $x_{A_t}\in\mathbb R^{A_t}$ is the restriction of vector $x\in\mathbb R^d$ to $A_t$.
The arity $|A_t|$ of function $f_t$ can be arbitrarily large, however 
we assume the existence of an efficient {\em min-oracle} that for a given vector $\lambda\in\mathbb R^{A_t}$ computes 
$x\in\argmin\limits_{x\in \dom f_t} \left[f_t(x)+\langle \lambda,x\rangle\right]$
together with the cost $f_t(x)$,
where  $\dom f_t=\{x\in\{0,1\}^{A_t}\:|\:f_t(x)<+\infty\}\ne\varnothing$ is the effective domain of $f_t$.
It will be convenient to denote
$$
h_t(\lambda)\!=\!\!\!\!\!\min\limits_{\stackrel{}{x\in \dom f_t}}\!\!\!\!\! \left[f_t(x)+\langle \lambda,x\rangle\right]
\!=\!\min\limits_{y \in \SY_t} \langle y, [\lambda\;1] \rangle
\!=\!\min\limits_{y \in \mathbb Y_t} \langle y, [\lambda\;1] \rangle
$$
where subsets $\SY_t,\Y_t\subseteq [0,1]^{A_t}\otimes\mathbb R$ are defined as follows: 
$$
\SY_t=\{[x\; f_t(x)]\::\:x\in \dom f_t\}\;\quad \Y_t=\conv(\SY_t)
$$
The assumption means that we can efficiently compute a supergradient of concave function $h_t(\lambda)$
at a given $\lambda\in\mathbb R^{A_t}$.

Since~\eqref{eq:decomposition} is in general an NP-hard problem,
our goal will be to solve a certain convex relaxation of~\eqref{eq:decomposition},
which will turn out to be equivalent to the {\em Basic LP relaxation} (BLP) of~\eqref{eq:decomposition}~\cite{kolmogorov15:power}.
This relaxation has been widely studied in the literature, especially for the MAP-MRF problem
(in which case it is usually called the {\em local polytope} relaxation~\cite{WainwrightJordan08,Werner07}).
We emphasize, however, that our methodology is different from most previous works:
before applying the BLP relaxation, we represent the objective as a function of Boolean
indicator variables. This allows expressing complicated combinatorial constraints
such as those in multi-label discrete tomography and graph matching problems (see Section~\ref{sec:applications}).

\myparagraph{Lagrangean relaxation}
For a vector $y\in \Y_t$ let us denote the first $|A_t|$ components as $y_\star\in[0,1]^{A_t}$
and the last one as $y_\circ\in\mathbb R$ (so that $y=[y_\star\; y_\circ]$). 
We also denote $\SY=\bigotimes_{t\in  T} \SY_t$ and $\Y=\bigotimes_{t\in  T} \Y_t=\conv(\SY)$. The $t$-th component of vector $y\in\Y$
will be denoted as $y^t\in \Y_t$.
Problem~\eqref{eq:decomposition} can now be equivalently written as
\begin{equation}
  \label{eq:decomposition1}
  \min_{\substack
{ y \in \Y\;,\; x\in\{0,1\}^d
\\y^t_\star = x_{A_t}\; \forall t\in T}
} \;\;
\sum_{t\in T} y^t_\circ
\end{equation}
We form the relaxation of~\eqref{eq:decomposition1} by removing the non-convex constraint $x\in\{0,1\}^d$:
\begin{equation}
  \label{eq:decomposition1:relaxation}
  \min_{\substack
{ y \in \Y\;,\; x\in\mathbb R^d
\\y^t_\star = x_{A_t}\; \forall t\in T}
} \;\;
\sum_{t\in T} y^t_\circ
\end{equation}
It can be shown that problem~\eqref{eq:decomposition1:relaxation} is equivalent to the BLP relaxation of~\eqref{eq:decomposition}, see~\cite{swoboda:2018arxiv}.

We will not directly optimize this relaxation, but its Lagrangean dual~\cite{SonGloJaa_optbook}.
For each equality constraint $y^t_\star=x_{A_t}$ we
introduce Lagrange multipliers $\lambda^t\in\mathbb R^{A_t}$. The collection
of these multipliers will be denoted as $\lambda\in\bigotimes_{t\in T}\mathbb R^{A_t}$.
The dual will be optimized over the set
\begin{equation}
\Lambda = \left\{ \lambda \:  
: \:
\sum\limits_{t\in T_i} \lambda^t_i = 0 \;\;\; \forall i\in[d]\right\}
\end{equation}
where we denoted 
$$T_i=\{t\in T\::\:i\in A_t\}\,.$$

\begin{proposition}
\label{prop:dual-of-decomposition}
The dual of~\eqref{eq:decomposition1:relaxation} w.r.t.\ the equality constraints $y^t_\star = x_{A_t}$ is
\begin{equation}
  \label{eq:dual-decomposition}
  \max_{ \lambda \in \Lambda} h(\lambda), \quad h(\lambda) := \sum_{t\in T} h_t(\lambda^t)
\end{equation}
Furthermore, the optimal values of problems~\eqref{eq:decomposition1:relaxation} and~\eqref{eq:dual-decomposition}
coincide. (This value can be $+\infty$, meaning that~\eqref{eq:decomposition1:relaxation} is infeasible and \eqref{eq:dual-decomposition} is unbounded).
\end{proposition}
Next, we describe how we maximize function~$h(\lambda)$.

\myparagraph{Proximal term}
We have a non-smooth concave maximization problem in $\lambda$, hence algorithms that require a differentiable objective functions will not work.
In proximal bundle methods~\cite{ProximityControlBundleMethodsKiwiel} an additional proximal term is added.
This results in the new objective
\begin{equation}
  \label{eq:dual-decomposition-trust-region}
  \max_{ \lambda \in \Lambda} 
  h_{\mu,c}(\lambda), \quad h_{\mu,c}(\lambda) :=
    h(\lambda)
  -
  \frac{1}{2 c}   \norm{\lambda - \mu}^2
\end{equation}
for a center point $\mu\in\Lambda$.
The proximal quadratic terms act as a trust-region term in the vicinity of $\mu$ and make the function strongly concave, hence smoothing the dual.
A successively refined polyhedral approximation~\cite{Lemarechal1986Bundle} is typically used for solving~\eqref{eq:dual-decomposition-trust-region}.
We develop a proximal method that will alternate between minimizing~\eqref{eq:dual-decomposition-trust-region} with the help of a multi-plane
block coordinate Frank-Wolfe method and updating the proximal center $\mu$.

\subsection{Maximizing $h_{\mu,c}(\lambda)$: BCFW algorithm}
Objectives similar to~\eqref{eq:dual-decomposition-trust-region}
 (without the summation constraint on the $\lambda$-variables) are used for training structural Support Vector Machines (SSVMs).
Following~\cite{BCFW,MP-BCFW,Osokin:ICML16}, we use a block-coordinate Frank-Wolfe algorithm (BCFW)
applied to the dual of~\eqref{eq:dual-decomposition-trust-region}, more specifically its multi-plane version MP-BCFW~\cite{MP-BCFW}.
The dual of~\eqref{eq:dual-decomposition-trust-region} is formulated below.
\begin{proposition}
\label{prop:dual-of-decomposition-plus-trust-region}
The dual problem to $\max_{\lambda \in \Lambda} h_{\mu,c}(\lambda)$ is
  \begin{multline}
  \label{eq:dual-of-decomposition-plus-trust-region}
    \min\limits_{y\in\Y }
    f_{\mu,c}(y),\\
f_{\mu,c}(y) :=
    \max\limits_{\lambda \in \Lambda}
\sum_{t\in T} 
\left[        \langle y^t, [ \lambda^t\; 1] \rangle
     - \frac{1}{2c} \norm{\lambda^t - \mu^t}^2 \right]
  \end{multline}
Define $\nu \in \R^d$ by
$
\nu_i = \frac{1}{|T_i|} \sum_{t\in T_i} (c \cdot y^t_i + \mu^t_i)
$
for $i \in [d]$.
Then the optimal $\lambda$ in~\eqref{eq:dual-of-decomposition-plus-trust-region} is
$
\lambda^t  = c \cdot y^t_\star + \mu^t  - \nu_{A_t} 
$
and 
\begin{eqnarray}
f_{\mu,c}(y) &=& \sum_{t\in T}\left( \frac{c}{2} \norm{y^t_\star}^2 + \langle y^t, [\mu^t\; 1]\rangle \right)
- \sum_{i=1}^d \frac{|T_i|}{2c}\nu_i^2 \nonumber \\
 \nabla_t f_{\mu,c}(y) &=& [\lambda^t\; 1]
\end{eqnarray}
where $\nabla_t$ denotes the derivative w.r.t.\ variables $y^t$.
\end{proposition}

Next, we review and adapt to our setting BCFW and MP-BCFW algorithms for minimizing
function $f_{\mu,c}(y)$ over $y\in\Y$.
We will also describe a practical improvement to the implementation, namely compact representation of planes,
and discuss the problem of estimating the duality gap.

\myparagraph{BCFW~\cite{BCFW}}
The algorithm maintains feasible vectors $y\in\Y$.
At each step BCFW tries to decrease the objective $f_{\mu,c}(y)$ by updating
component $y^t$ for a chosen term $t$ (while maintaining feasibility).
To achieve this, it first linearizes the objective by using the Taylor expansion around the current point:
\begin{equation}
f_{\mu,c}(z) 
\approx  \langle \nabla f_{\mu,c}(y),z \rangle+{\tt const}\,.
\end{equation}
The optimal solution $z^t \in \Y_t$ of the linearized objective is computed by calling the $t$-th oracle:
$z^t \leftarrow\argmin\limits_{z^t \in \SY_t} \langle \nabla_t f_{\mu,c}(y),z^t \rangle$.
The new vector $y$ is obtained as the best interpolation of $y^t$ and $z^t$ with all other components $s \neq t$ fixed to $y^s$,
i.e.\ $y^s(\gamma) \leftarrow \left\{ \begin{array}{ll} y^s, & s \neq t \\ (1-\gamma)y^t + \gamma z^t,& s = t \end{array} \right.$.
The step size $\gamma \in [0,1]$ is chosen to minimize the objective.
The optimal $\gamma$ can be easily computed in closed form (see Lemma~\ref{lemma:step-size-BCFW} in~\cite{swoboda:2018arxiv}). 
One pass of BCFW is summarized in Algorithm~\ref{alg:BCFW}.
To avoid the expensive recomputation of the sum $\nu$ in Prop.~\ref{prop:dual-of-decomposition-plus-trust-region} 
needed for computing the gradient and the step size, it is held and updated explicitly in Algorithm~\ref{alg:BCFW}.

\begin{algorithm}[t]\small
  \caption{One pass of BCFW. Input: vectors $y\in\Y$, $\mu\in\Lambda$ and $\nu\in\mathbb R^d$ computed as in Prop.~\ref{prop:dual-of-decomposition-plus-trust-region}.
}\label{alg:BCFW}
\begin{algorithmic}[1]
  \STATE {\bf for each} $t\in T$ {\bf do} in a random order

  \STATE~~set $\lambda^t= c \cdot y^t_\star + \mu^t - \nu_{A_t}$ 

  \STATE~~call $t$-th oracle for $\lambda^t$: \ \ $z^t\leftarrow\argmin\limits_{z^t \in \SY_t} \langle z^t,[\lambda^t\; 1] \rangle$
\\ ~~ \big(i.e.\ let $x\leftarrow\argmin\limits_{x\in\SX_t}[f_t(x)+\langle\lambda^t,x\rangle]$ and $z^t=[x\;f_t(x)]$\big)
  \STATE~~interpolate $y(\gamma)^s \leftarrow \left\{ \begin{array}{ll} y^s, & s \neq t \\ (1-\gamma)y^t + \gamma z^t,& s = t \end{array} \right.$
  \STATE~~compute $\gamma\leftarrow \argmin_{\gamma\in[0,1]} f_{\mu,c}(y(\gamma))$: \\ ~~set 
$\gamma\leftarrow
\frac{\langle [\lambda^t\; 1] , z^t - y^t \rangle}
{c \norm{y^t_\star - z^t_\star}^2}
$
      and clip $\gamma$ to $[0,1]$ 
  \STATE~~set $\nu_i \leftarrow \nu_i + \frac{c}{|T_i|} (y(\gamma)^t_i - y^t_i)$ for $i \in A_t$ and $y^t \leftarrow y(\gamma)^t$ \!\!\!\!
  \STATE {\bf end for}
\end{algorithmic}
\end{algorithm}

\myparagraph{MP-BCFW~\cite{MP-BCFW}} In this paper we use the multi-plane version of BCFW.
This method caches planes~$z^t$ returned by min-oracles for terms $h_t(\lambda^t)=\min_{z^t \in \SY_t} \langle z^t, [\lambda^t\ 1] \rangle$.
Let $\tilde \SY_t\subset \SY_t$ be the set of planes currently stored in memory for the $t$-th subproblem. 
It defines an approximation
$\tilde h_t(\lambda^t)=\min_{z^t \in \tilde\SY^t} \langle z^t, [\lambda^t \ 1] \rangle$ of term 
$h_t(\lambda^t)$. 
Note that $\tilde h_t(\lambda^t)\ge h_t(\lambda^t)$ for any~$\lambda^t$.

MP-BCFW uses {\em exact passes} (that call the ``exact'' oracle for $h_t(\lambda^t)$ in line 3 of Algorithm~\ref{alg:BCFW})
and {\em approximate passes} (that call the ``approximate'' oracle for $\tilde h_t(\lambda^t)$).
One MP-BCFW iteration consists of one exact pass followed by several approximate passes.
The number of approximate passes is determined automatically by monitoring how
fast the objective decreases per unit of time.
Namely, the method computes the ratio
$\frac{f_{\mu,c}(y^\circ)-f_{\mu,c}(y)}{\Delta t}$ where $y^\circ$ is the vector at the beginning of the MP-BCFW iteration,
$y$ is the current vector and $\Delta t$ is the time passed since the beginning of the MP-BCFW iteration.
If this ratio drops after an approximate pass then the iteration terminates,
and a new MP-BCFW iteration is started with an exact pass.
The number of approximate passes will thus depend on the relative speed of exact and approximate oracles.

Note that the time for an approximate oracle call is proportional to the size of the working set $|\tilde \SY_t|$.
The method in~\cite{MP-BCFW} uses a common strategy for controlling this size:
planes that are not used during the last $K$ iterations are removed from $\SY_t$.
We use $K=10$, which is the default parameter in~\cite{MP-BCFW}.

\myparagraph{Compact representation of planes}
Recall that planes in the set $\tilde\SY_t$ have the form
$[x\;f_t(x)]$ for some $x\in\dom f_t\subseteq\{0,1\}^{A_t}$.
A naive approach (used, in particular, in previous works~\cite{MP-BCFW,Osokin:ICML16})
is to store them explicitly as vectors of size $|A_t|+1$.
We observe that in some applications a special structure of $\SX_t$ 
allows storing and manipulating these planes more efficiently.
For example, in MAP-MRF inference problems a variable with $k$ possible
values can be represented by a single integer, rather than $k$ indicator variables.

In our implementation we assume that each vector $x\in\dom f_t$ can be represented
by an object $s$ in a possibly more compact space $\CX_t$. To specify term $f_t$,
the user must provide an array $\map_t:[|A_t|]\rightarrow[d]$ that determines
set $A_t\subseteq[d]$ in a natural way, specify the size of objects $s\in\CX_t$, and implement the following functions:
\begin{itemize}
\item[F1.] A bijection $\sigma_t:\CX_t\rightarrow\SX_t$.

\item[F2.] Min-oracle that for a given $\lambda_t$ computes $x\in\argmin\limits_{x\in\dom f_t}[f_t(x)+\langle\lambda^t,x\rangle]$
and returns its compact representation $s$ (i.e.\  $\sigma_t(s)\!=\!x$) together with the cost~$f_t(x)$.

\item[F3.]  A function that computes inner product $\langle  \lambda^t, \sigma_t(s) \rangle$
 for a given vector $\lambda^t$ and object $s\in\CX_t$.
\end{itemize}

Note, 
calling the approximate oracle in line 3 involves calling the function in (F3) $|\tilde\SY_t|$ times; 
it typically takes $O(|\tilde\SY_t|\cdot {\tt size}_t)$ time where ${\tt size}_t$ is the length of the array for storing $s\in\CX_t$.

\begin{remark}
The efficient plane storage mechanism gives roughly a $25\%$ speedup of a single MP-BCFW pass on the \textbf{protein-folding} MRF dataset (see Sections~\ref{sec:applications} and~\ref{sec:experiments}).
Moreover, it typically results in a slightly better objective value obtained after each MP-BCFW pass, since more approximate passes can be done before an exact pass is called (since approximate passes are accelerated by the compact plane storage, their objective decrease per unit of time is higher, hence they are called more often).
\end{remark}

\subsection{Algorithm's summary}
Recall that our goal is to maximize function $h(\lambda)$ over $\lambda\in\Lambda$;
this gives a lower bound on the original discrete optimization problem $\min_{x\in\SX}f(x)$.
We now summarize our algorithm for solving $\max_{\lambda\in\Lambda}h(\lambda)$,  and describe our choices of parameters.
To initialize, we set $\mu^t={\bf 0}$, $y^t\leftarrow\argmax_{y^t\in\SY_t}\langle y^t,[\mu^t\; 1]\rangle$ and $\tilde\SY=\{y^t\}$ for each $t\in T$.
Then we start the main algorithm. After every 10 iterations of MP-BCFW
we update $\mu\leftarrow\lambda^\ast$ (keeping vectors $y^t$ and sets $\tilde\SY_t$ unchanged),
where $\lambda^\ast$ is the vector with the largest value of objective
$h(\lambda^\ast)$ seen so far. Since evaluating $h(\lambda)$ 
is an expensive operation, we do it for the current vector $\lambda$ only after every 5 iterations of MP-BCFW.
\begin{remark}[Convergence]
If we evaluated the inner iterations in MP-BCFW exactly, our method would amount to the proximal point algorithm which is convergent~\cite{rockafellar1976monotone}.
Even with non-exact evaluation, convergence can be proved when the error in the evaluation of the proximal is shrinking fast enough towards zero.
However, we have use a more aggressive scheme that updates the proximal point every 10 iterations and does not bound the inexactness of the proximal step evaluation.
Experimentally, we have found that it gives good results w.r.t.\ the objective of the overall problem~\eqref{eq:dual-decomposition}.
\end{remark}

\subsection{Estimating duality gap}\label{sec:duality-gap}
To get a good stopping criterion, it is desirable to have
a bound on the gap $h(\lambda)-h(\lambda^\ast)$ between
the current and optimal objectives. This could be easily
done if we had feasible primal and dual solutions.
Unfortunately, in our case vector $y\in\Y$ is not a feasible
solution of problem~\eqref{eq:decomposition1:relaxation},
since it does not satisfy equality constraints $y^t_\star=x_{A_t}$. 
\footnote{
We say that vector $y$ is a feasible (optimal) solution of~\eqref{eq:decomposition1:relaxation} if there exists a vector $x\in\mathbb R^d$
so that $(y,x)$ is a feasible (optimal) solution of~\eqref{eq:decomposition1:relaxation}. Clearly, $x$ can be easily computed from feasible $y$,
so omit it for brevity.
}
To get a handle on the duality gap, we propose to use the following quantities:
\begin{equation}
\label{eq:duality-gap-quantities}
A_{y,\lambda}\!=\!\sum_{t\in T}\langle y^t,[\lambda^t\; 1]\rangle-h(\lambda) \;,\;
B_y\!=\!\sum_{i=1}^d \left[\max_{t\in T_i}y^t_i - \min_{t\in T_i}y^t_i\right]
\end{equation}
\begin{proposition}\label{prop:AB}
Consider pair $(y,\lambda)\in\Y\!\times\!\Lambda$. \\
(a) There holds $A_{y,\lambda}\ge 0$, $B_y\ge 0$
and 
\begin{equation}\label{eq:gap-bound}
h(\lambda^\ast)-h(\lambda)\le A_{y,\lambda} + B_y\cdot \norm{\lambda^\ast-\lambda}_{1,\infty}\quad\forall  \lambda^\ast\in\Lambda
\end{equation}
where we denoted $\norm{\delta}_{1,\infty}=\max_{i\in[d]} \sum_{t\in T_i}|\delta^t_i|$. \\
(b) We have $A_{y,\lambda}=B_y=0$ if and only if $y$ and $\lambda$ are optimal solutions of problems~\eqref{eq:decomposition1:relaxation}
and~\eqref{eq:dual-decomposition}, respectively.
\end{proposition} 
Note, if we knew that an optimal solution $\lambda^\ast\in\max_{\lambda\in\Lambda}h(\lambda)$
belongs to some bounded region then we could use~\eqref{eq:gap-bound}
to obtain a bound on the duality gap. Such region can be obtained for some applications,
but we did not pursue this direction.

\section{Applications}
\label{sec:applications}
In this section we give a detailed description of the three applications used in the evaluation:
Markov Random Fields (MRFs), discrete tomography and the graph matching problem.
The latter two are both extensions of the MAP-inference problems for MRFs.
Those three problems are reviewed below.

\subsection{Markov Random Fields}
An MRF consists of a graph $G=(V,E)$ and a discrete set of labels $\SX_v$ for each $v \in V$.
The goal in Maximum-A-Posteriori (MAP) inference is to find labeling $(x_v)_{v \in V} \in \bigotimes_{v\in V} \SX_v =: \SX$ that is minimal with respect to the potentials:
\begin{equation}
  \label{eq:MAP-inference}
  \min_{x \in \SX} f(x), \; f(x) := \sum_{v \in V} \theta_v(x_v) + \sum_{uv \in E} \theta_{uv}(x_u,x_v)\,.
\end{equation}
This problem is NP-hard for general graphs $G$, but can be 
solved efficiently for trees.
Hence, we choose a covering $G$ by trees, as done in~\cite{DualDecompositionKomodakis}.
Additionally, we seek a small number of trees, such that the number of Lagrangean variables stays small and optimization becomes faster.

\myparagraph{Arboricity}
A tree covering of a graph is called a \emph{minimal tree cover}, if there is no covering consisting of fewer trees.
The associated number of trees is called the graph's \emph{arboricity}.
We compute the graph's arboricity together with a minimal tree cover efficiently with the method~\cite{ForestsFramesAndGamesGabw1992}

\myparagraph{Boolean encoding}
To phrase the problem as an instance of~\eqref{eq:decomposition}, we encode labelings
$x\in\SX$ via indicator variables $x_{i;a}=[x_i=a]\in\{0,1\}$ for $i\in V,a\in\SX_i$
while adding constraints $\sum_a x_{i;a}=1$ (i.e.\ assigning infinite cost
to configurations that do not satisfy this constraint).

A tree cover and a Boolean encoding are also used for the two problems below;
we will not explicitly comment on this anymore.

\subsection{Discrete tomography}

\begin{figure}
\begin{minipage}{0.97\columnwidth}
\centering
\includegraphics[width=0.8\textwidth]{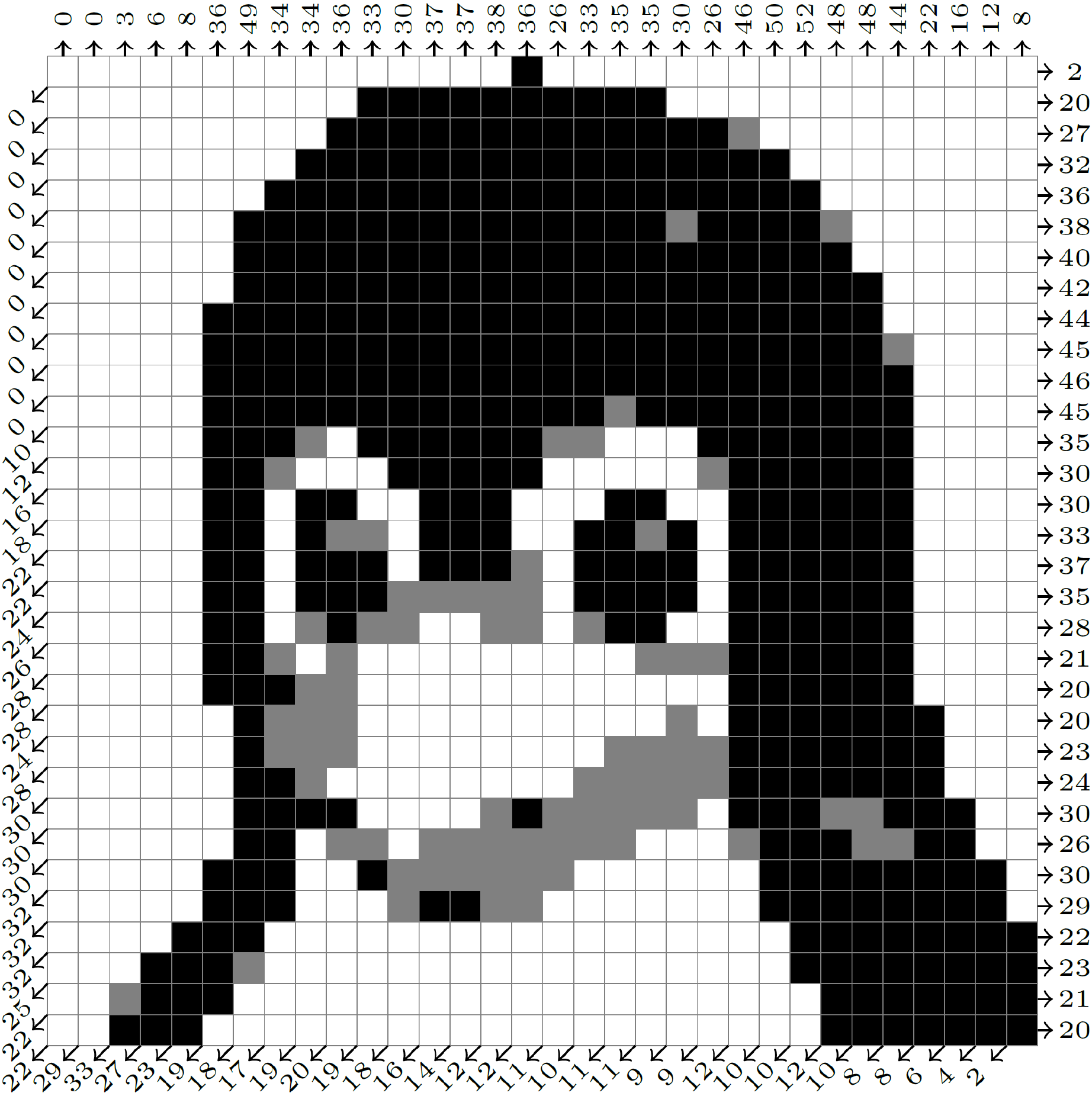} 
\caption{
Illustration of a discrete tomography problem.
Image intensity values are 0 (white), 1 (gray) and 2 (black).
Small arrows on the side denote the three tomographic projection directions (horizontal, vertical, and diagonal). 
Values at arrow heads denote the intensity value along tomographic projections.
}
\label{fig:dt-illustration}
\end{minipage}
\hspace{10pt}
\begin{minipage}{0.97\columnwidth}
\centering
\includegraphics[width=0.8\textwidth]{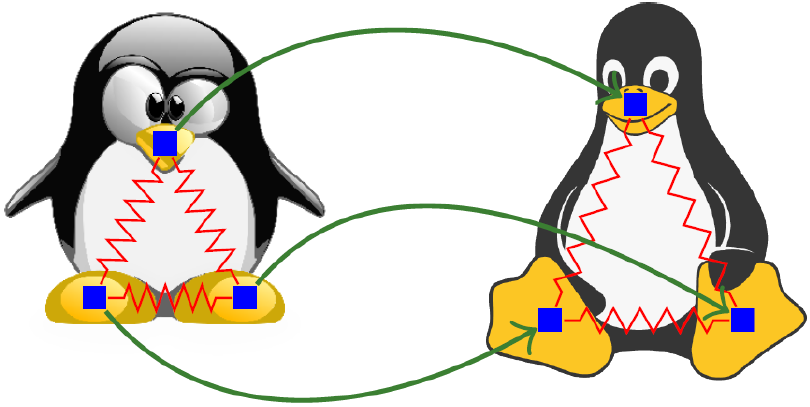} 
\caption{
Illustration of a graph matching problem matching nose and left/right feet of two penguins.
The \textcolor{blue}{blue nodes} on the left penguin correspond to the underlying node set $V$, while the \textcolor{blue}{blue nodes} on the right penguin correspond to the labels $\SL$.
The \textcolor{green}{green lines} denotes the matching. 
Note that no two labels are matched twice.
The \textcolor{red}{red springs} denote pairwise costs $\theta_{ij}$ that encourage geometric rigidity of the matching.
}
\label{fig:graph-matching}
\end{minipage}
\end{figure}

The discrete tomography problem is to reconstruct an image from a set of linear (tomographic) projections taken at different angles, where image intensity values are restricted to be from a discrete set of intensities.
See Figure~\ref{fig:dt-illustration} for an illustration.
The problem is ill-posed, since the number of linear projections is smaller than the number of pixels to reconstruct. 
Hence, we use a regularizer to penalize non-regular image structures.
Formally, we have an MRF $G = (V,E)$, where the node set corresponds to the pixel grid and the edges connect nodes which correspond to neighboring pixels.
The label space is $\SX_v = \{0,1,\ldots,k\}$ for some $k \in \N$.
Additionally, the labeling $x \in \SX$ must satisfy linear projections $Ax=b$ with $A \in \{0,1\}^{\abs{V} \times l}$.  
Usually, no local information is available, hence the unary potentials are zero: $\theta_v \equiv 0$ for $v \in V$.
The problem reads
\begin{equation}
  \label{eq:dt_energy}
  \min_{x \in \SX } f(x), \quad f(x) := \sum_{ij \in E} \theta_{ij}(x_i,x_j)
\end{equation}
where $\SX = \{x\in\{0,1,\ldots,k\}^V\::\:Ax=b\}$.
A typical choice for the pairwise potentials $\theta_{ij}$ is the truncated $L_1$-norm $\theta_{ij}(x_i,x_j) = \min(\abs{x_i-x_j},c)$. 
Each row of the projections $Ax=b$ forms another subproblem.
The $i$-th row of $Ax=b$ hence is of the form $\sum_{v \in V: A_{iv}=1} x_v = b_i$.
Efficient solvers for this problem were recently considered in~\cite{SSVM_DT_2017}.
We follow their recursive decomposition approach for the solution of the projection subproblems. Details are given below.

\paragraph{Discrete tomography subproblems}
We use a simplified version of the recursive decomposition approach of~\cite{SSVM_DT_2017} for efficiently solving the summation constraints $Ax = b$ of the discrete tomography problem.
Below we give the corresponding details.
Let the $i$-th row of $Ax=b$ be of the form $\sum_{v \in V: A_{iv}=1} x_v = b_i$ and recall that $x_v \in \{0,1,\ldots,k\}$.
Each such tomographic projection will correspond to a single subproblem.
Taking care of Lagrangean multipliers $\lambda$, we can rename variables and rewrite the problem as
\begin{equation}
\label{eq:summation-problem}
\min_{x_1,\ldots,x_n \in \{0,\ldots,k\}^n} \sum_{i=1}^n \sum_{l=0}^k \lambda_i(l) \cdot \eins_{x_i = l} \quad \text{s.t.} \quad \sum_{i=1}^n x_i = b
\end{equation}
We will follow a recursive decomposition approach.
To this end, we introduce helper \emph{summation variables} $s_{i:j} = \sum_{u=i}^j x_u$.
\paragraph{Variable partitions}
We partition the set $[1,n]$ into $\Pi_{1:\lfloor\frac{n}{2}\rfloor} = \{x_1,\ldots,x_{\lfloor\frac{n}{2}\rfloor} \}$ and $\Pi_{\lfloor\frac{n}{2}\rfloor+1:n} = \{x_{\lfloor\frac{n}{2}\rfloor+1},\ldots,x_n \}$.
We recursively partition $\Pi_{1:\lfloor\frac{n}{2}\rfloor}$ and $\Pi_{\lfloor\frac{n}{2}\rfloor+1:n}$ analoguously until reaching single variables.
This results in a tree with $\Pi_{1:n}$ as root.
\paragraph{Constraints on summation variables}
Whenever we have partitions $\Pi_{i:j} = \Pi_{i:k} \cup \Pi_{k+1:j}$, we add the constraint $s_{i:j} = s_{j:k} + s_{k+1:j}$.
\paragraph{Solving~\eqref{eq:summation-problem}}
We will propose a dynamic programming approach to solving~\eqref{eq:summation-problem}.
First, for each value $l$ of summation variable $s_{i:j}$ we store a value $\phi_{i:j}(l) \in \R$.
We compute $\phi_{i:j}$ recursively from leaves to the root.
For partitions $\Pi_{i:j} = \Pi_{i:k} \cup \Pi_{k+1:j}$ we compute 
\begin{equation}
\label{eq:message-computation}
\phi_{i:j}(l) = \min_{l' = 0,\ldots,l} \phi_{i:k}(l') + \phi_{k+1,j}(l-l') \quad \forall l\,.
\end{equation}
After having computed $\phi_{1:n}$, we set $s^*_{1:n} = b$.
Subsequently, we make a pass from root to leaves to compute the optimal label sum for each variable $s_{i:j}$ as follows: 
\begin{equation}
s^*_{i:k}, s^*_{k+1:j} = \min_{s_{i:k} + s_{k+1:j} = s^*_{i:j}} \phi_{i:k}(s_{i:k}) + \phi_{k+1:j}(s_{k+1:j})\,.
\end{equation}

\paragraph{Fast dynamic programming}
Naively computing~\eqref{eq:message-computation} needs $O(((j-i)\cdot k)^2)$ steps.
However, we use an efficient heuristic~\cite{FastMaximumConvolutionBussieck} that tends to have subquadratic complexity in practice.

\subsection{Graph matching}

The graph matching problem consists of finding a MAP-solution in an MRF $G=(V,E)$ where the labels for each node come from a common universe $\SX_v \subset \SL$ $\forall v \in V$.
The additional \emph{matching constraint} requires that two nodes cannot take the same label: $x_u \neq x_v$ $\forall  u,v \in V, u \neq v$.
Hence, any feasible solution defines an injective mapping into the set of labels.
For an illustration see Figure~\ref{fig:graph-matching}.
The problem is
\begin{equation}
  \label{eq:graph-matching}
  \min_{x \in \SX} f(x) \quad \text{ s.t. } \quad x_u \neq x_v \quad \forall u \neq v \,. 
\end{equation}
We use a minimum cost flow solver for handling the matching constraint, see e.g.~\cite{TorresaniEtAlQAP,HungarianBP,GraphMatchingMessagePassing} for an explanation of the minimum cost flow solver construction.

\vspace{-4pt}
\section{Experiments}
\label{sec:experiments}

\begin{figure*}[t]%
\centering%
\begin{tabular}{cccc}
\includegraphics[scale=0.22]{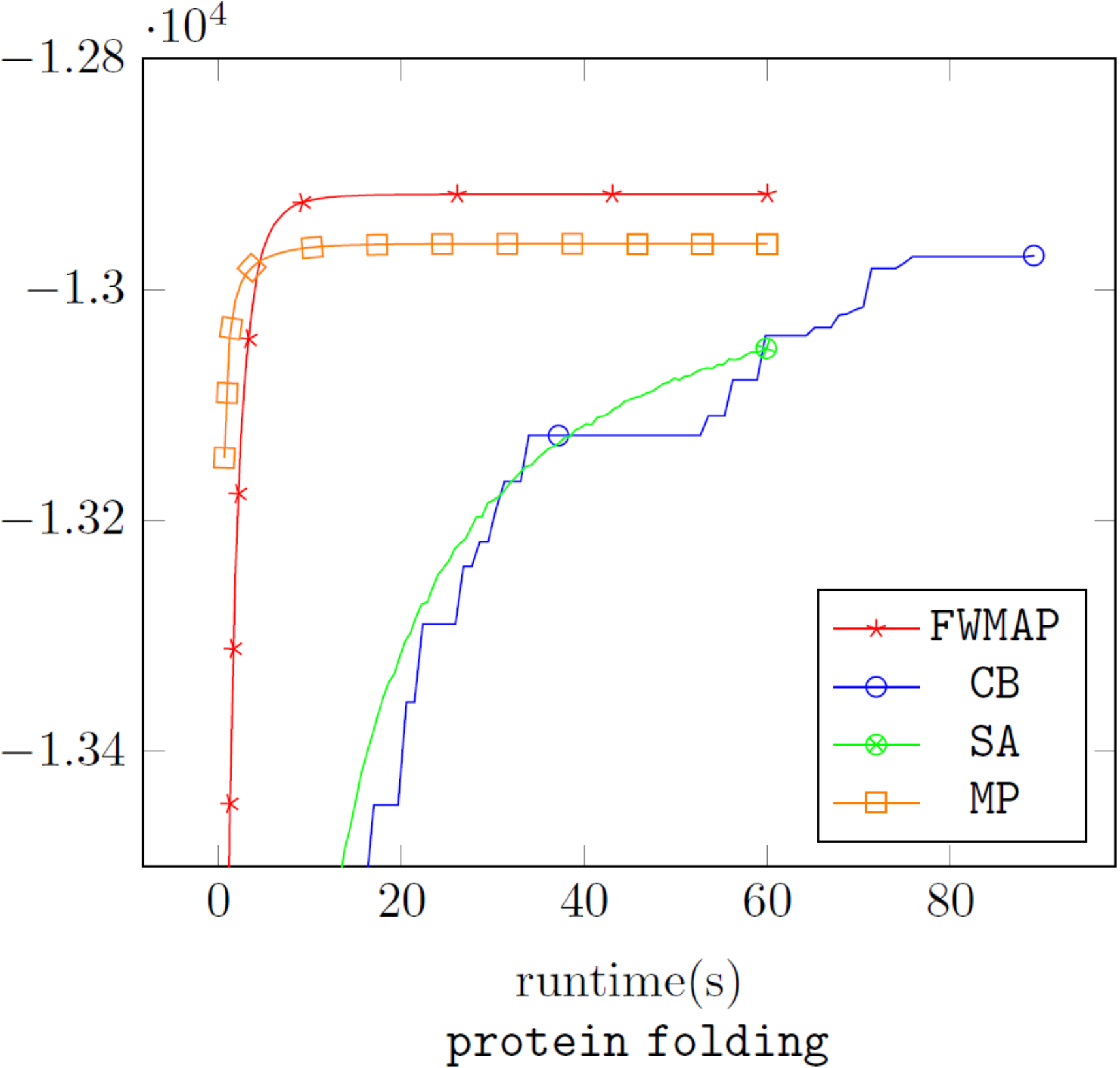} &%
\includegraphics[scale=0.22]{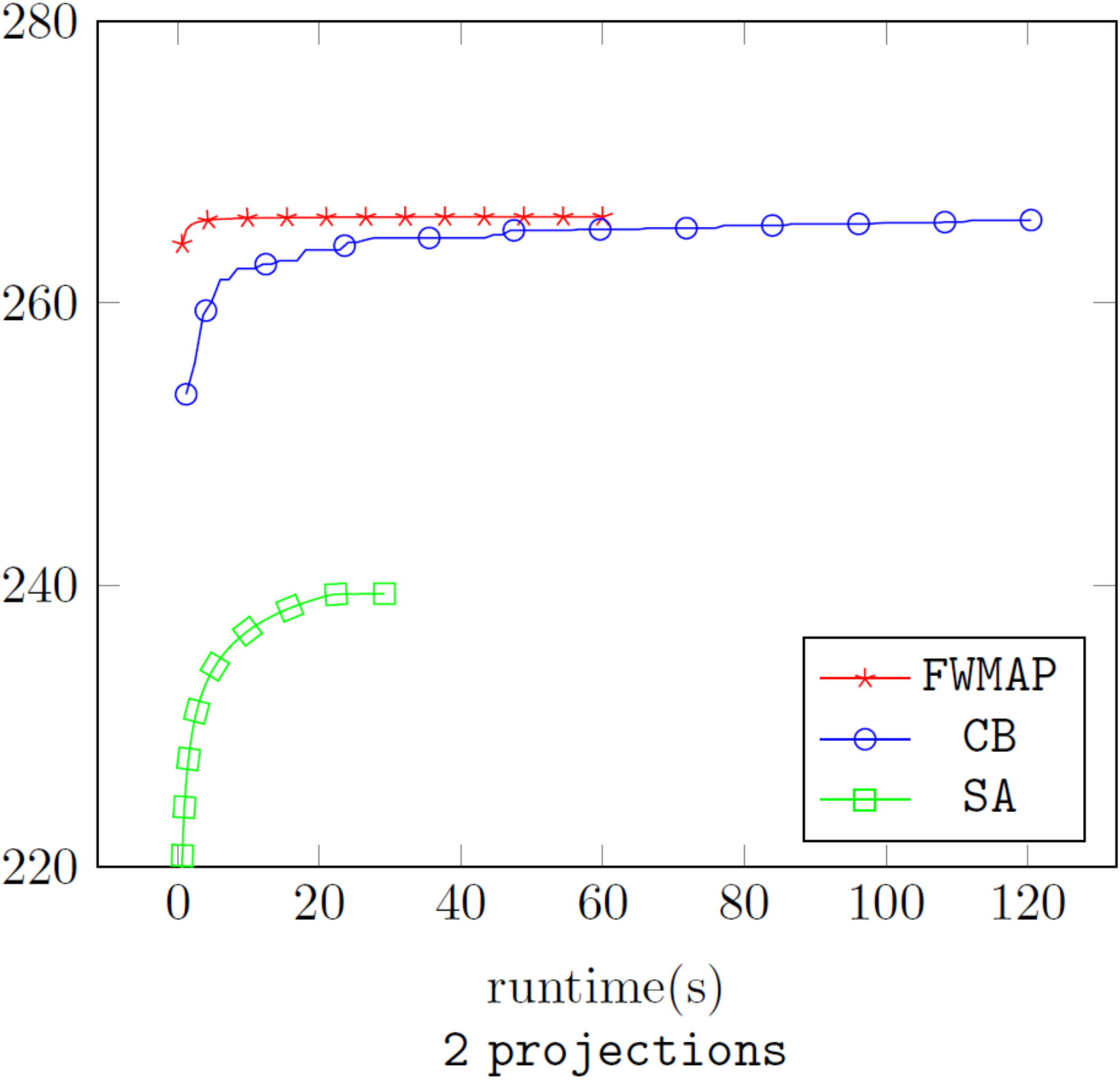} &%
\includegraphics[scale=0.22]{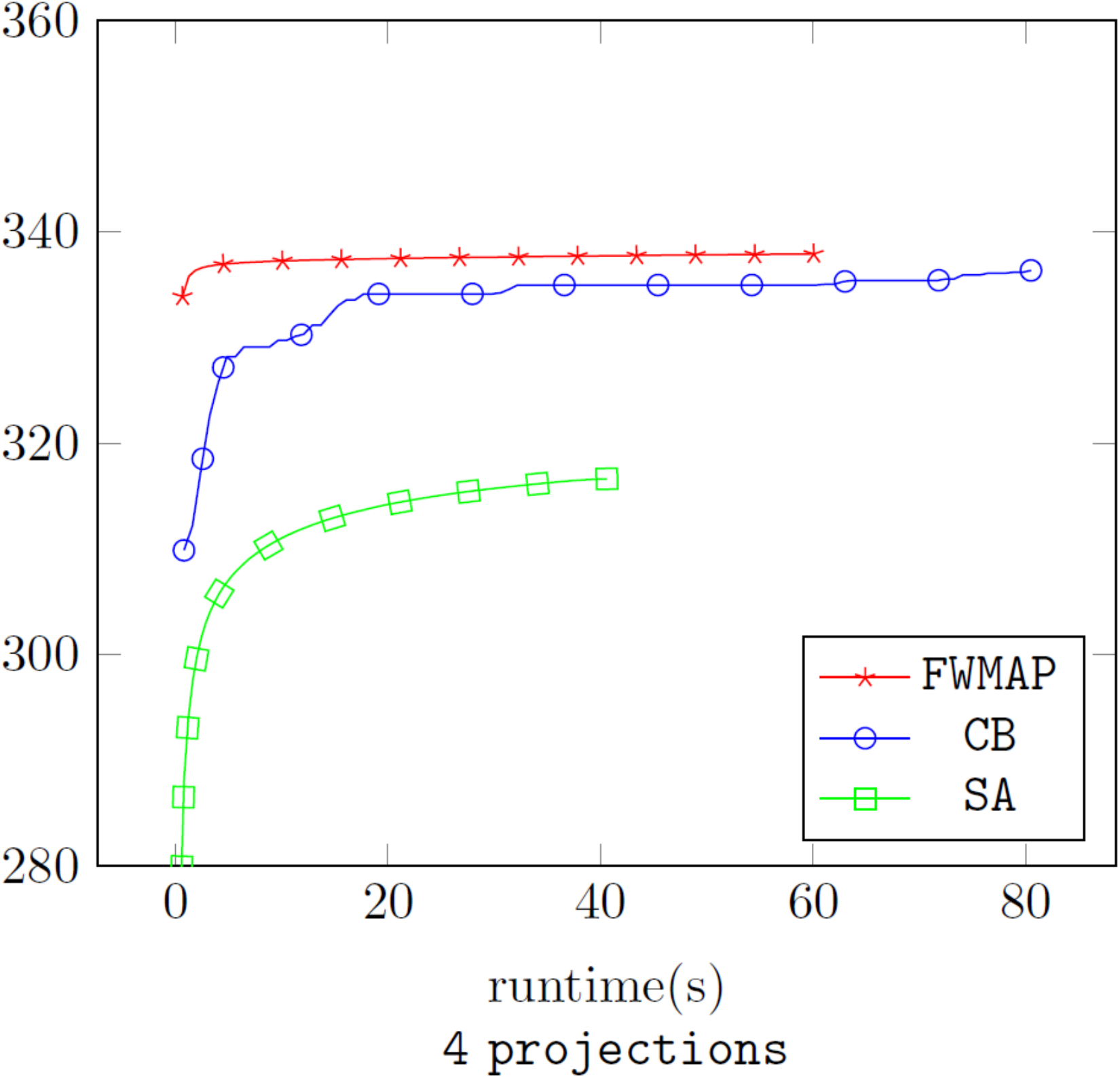} &%
\includegraphics[scale=0.22]{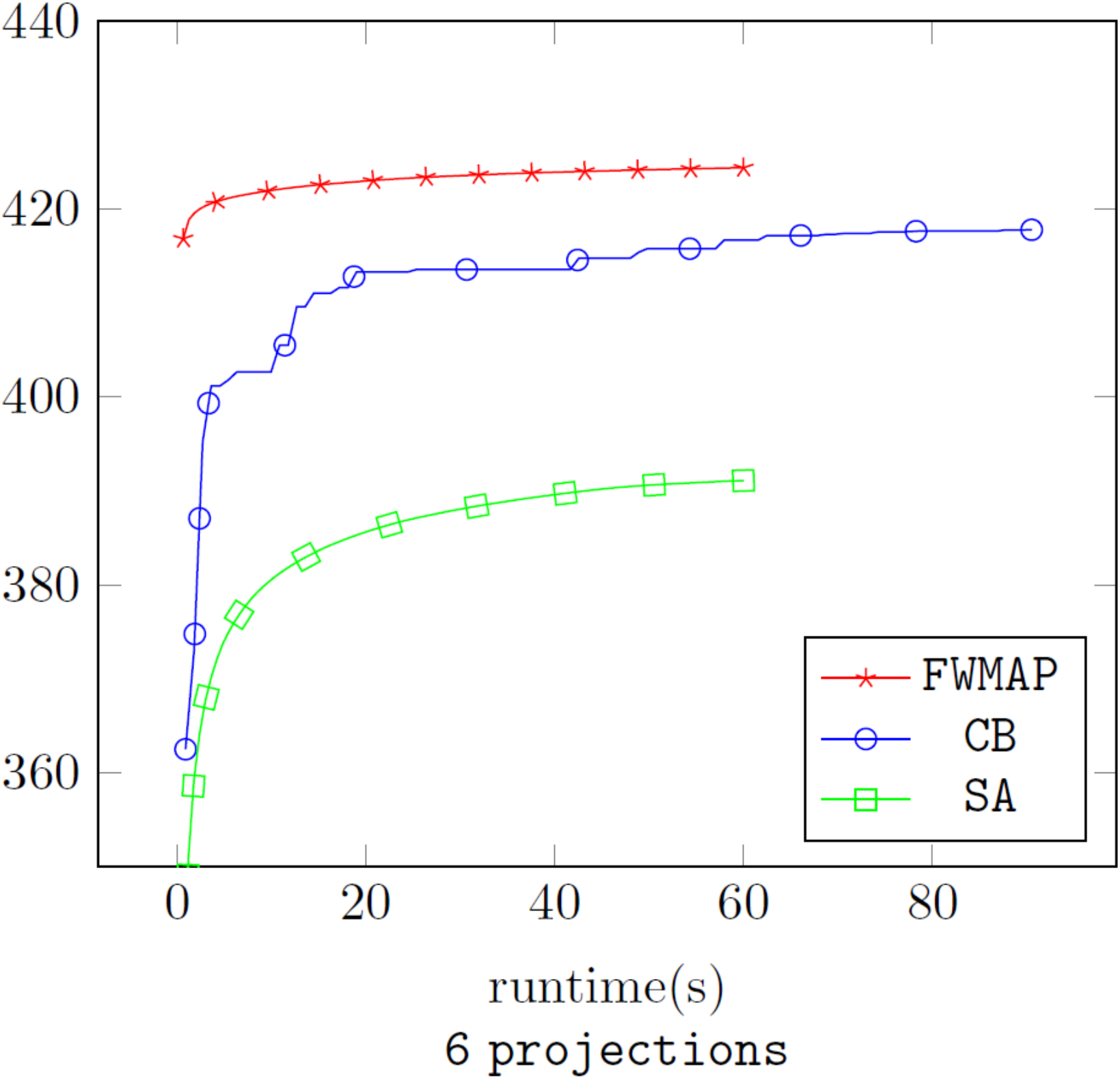} \\%
\includegraphics[scale=0.22]{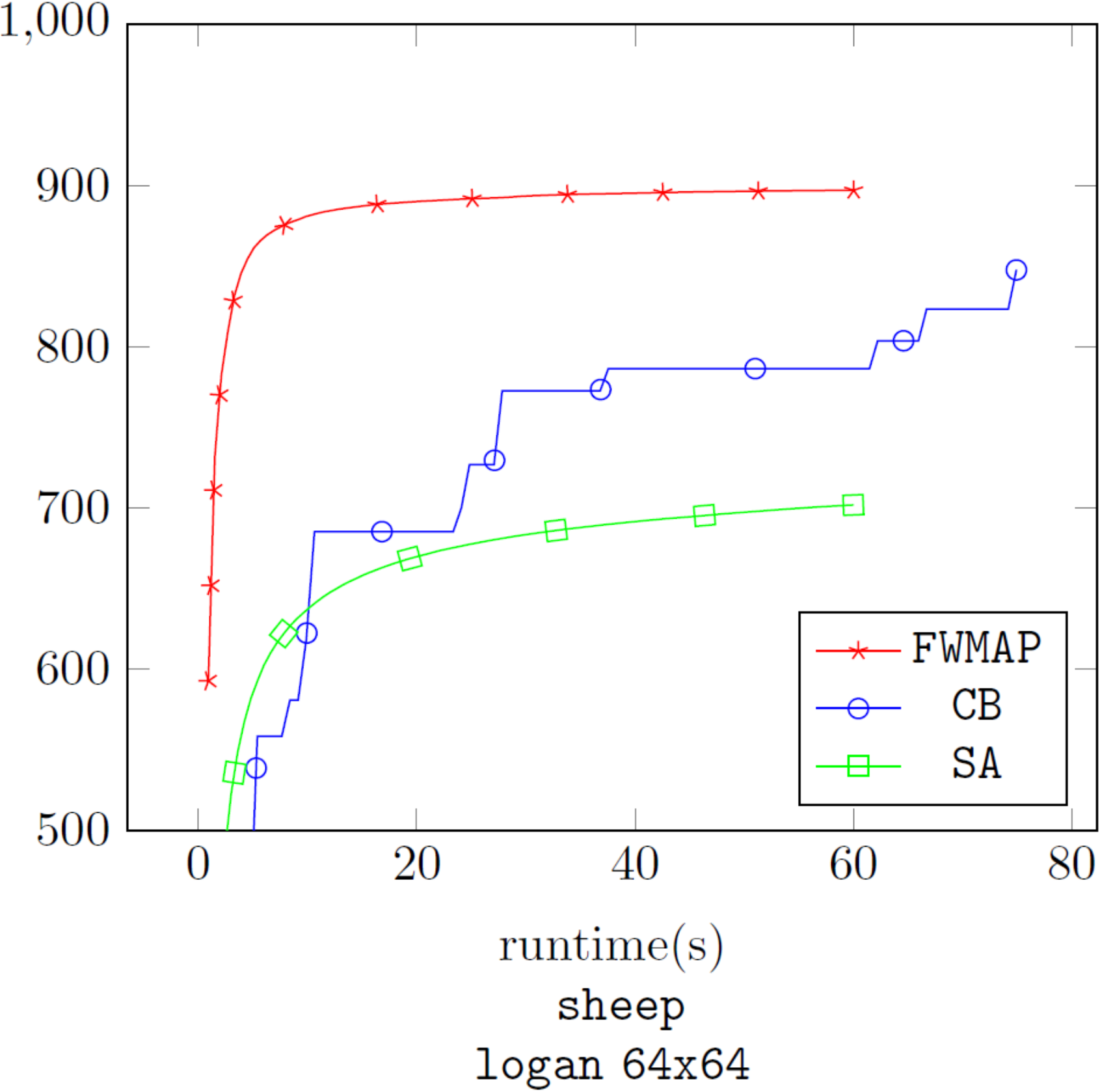} &%
\includegraphics[scale=0.22]{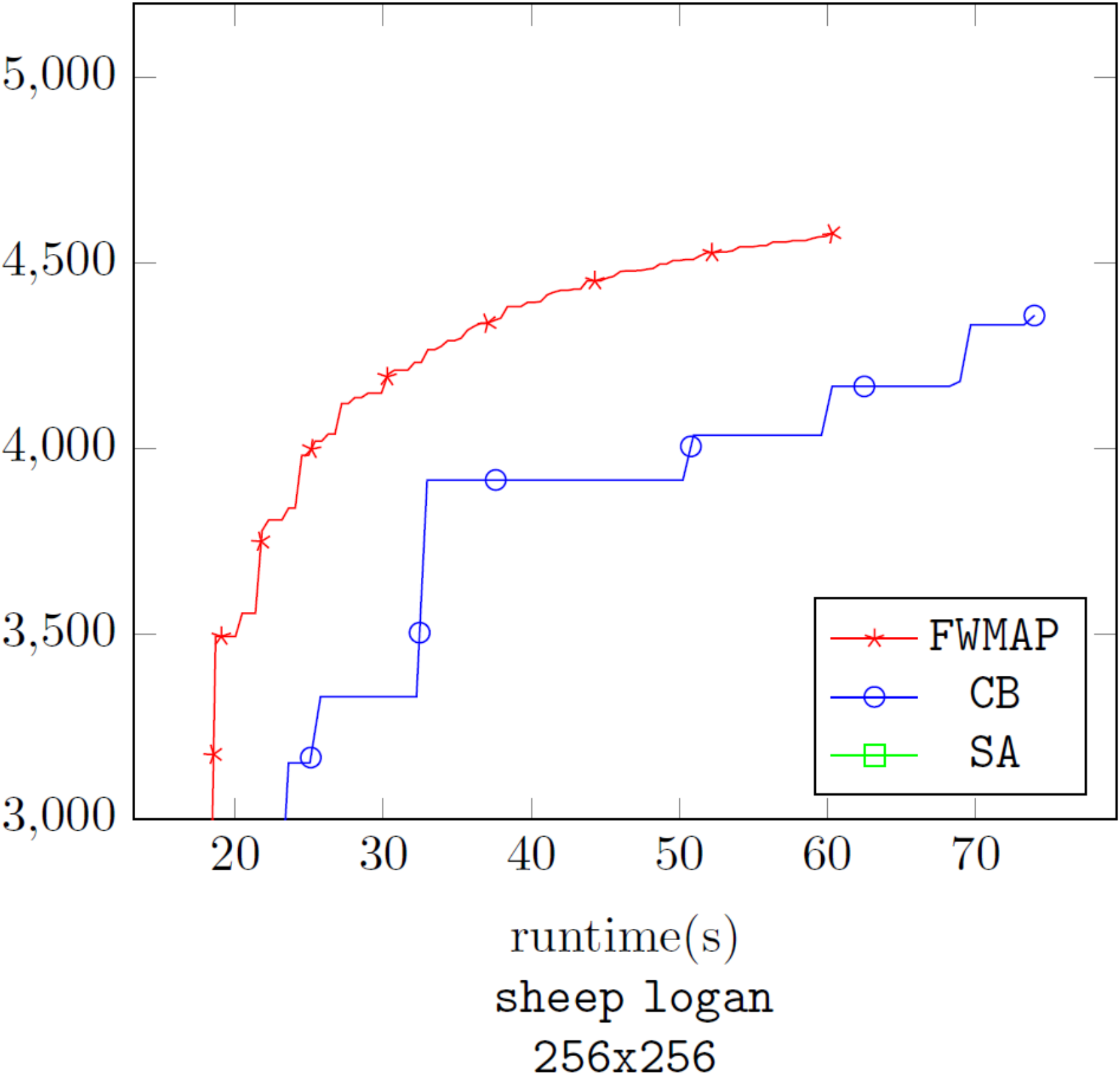} &%
    \smash{\raisebox{5pt}{
\includegraphics[scale=0.22]{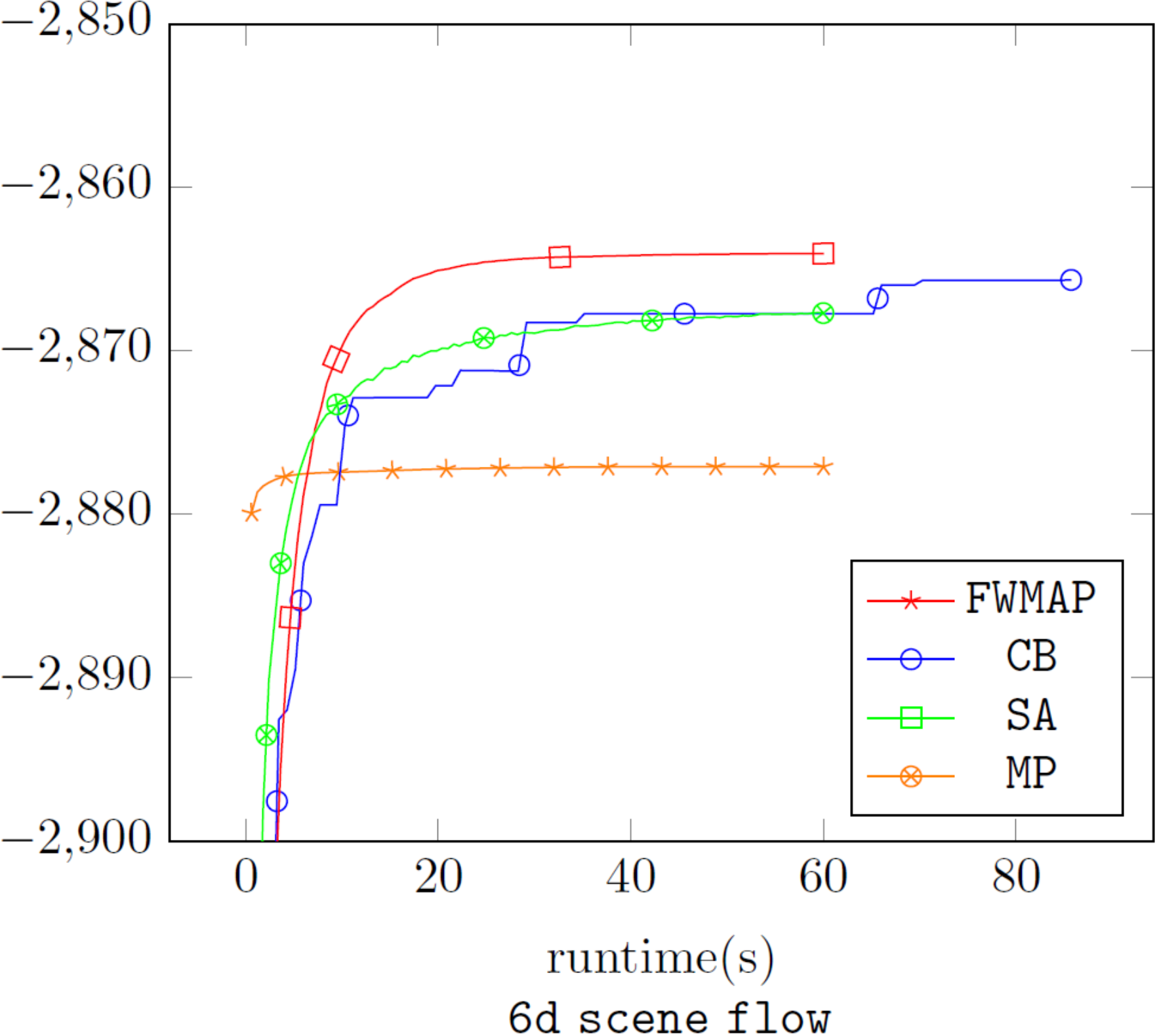} %
}}&
\\
\end{tabular}
  \caption{
    Averaged lower bound vs. runtime plots for the \textbf{protein folding} MRF dataset, discrete tomography (\textbf{synthetic} images with $2$, $4$ and $6$ projections, \textbf{sheep logan} image of sizes $64\times 64$ and $256\times 256$ with $2$, $4$ and $6$ projections), and the \textbf{6d scene flow} graph matching dataset.
  Values are averaged over all instances of the dataset.
  }
\label{fig:dataset_runtime_plots}
\end{figure*}

\begin{table*}
\centering
\begin{tabular}{|c@{\hspace{2pt}}|@{\hspace{2pt}}c@{\hspace{4pt}}c@{\hspace{4pt}}c@{\hspace{2pt}}|@{\hspace{2pt}}c@{\hspace{4pt}}c@{\hspace{4pt}}c@{\hspace{4pt}}c|}
\hline
\textbf{Dataset} & \# I & $\abs{V}$ & $\abs{E}$ & \texttt{FWMAP} & \texttt{CB} & \texttt{SA} & \texttt{MP} \\ \hline
\underline{MRF} &&&&&&& \\
\textbf{protein folding} & 11 & 33-40 & 528-780 & \textbf{ -12917.44 } & -12970.61 & -12960.30 & -13043.67 \\ \hline
\underline{Discrete tomography} &&&&&&&\\
\textbf{synthetic 2 proj.} & 9 & 1024 & 1984 & \textbf{266.12} &  265.89 & 239.39 & $\dagger$ \\ 
\textbf{synthetic 4 proj.} & 9 & 1024 & 1984 & \textbf{ 337.88 } & 336.33 & 316.61 & $\dagger$ \\ 
\textbf{synthetic 6 proj.} & 9 & 1024 & 1984 & \textbf{424.36}&  417.76 & 391.09 & $\dagger$ \\ 
\textbf{sheep logan} $\mathbf{64\times 64}$ & 3 & 4096 & 8064 &\textbf{ 897.18 } & 847.87 & 701.93 & $\dagger$ \\ 
\textbf{sheep logan} $\mathbf{256\times 256}$ & 3 & 65536 & 130560 &\textbf{ 4580.06 } & 4359.24 & 370.63 & $\dagger$ \\ \hline
\underline{Graph matching} &&&&&&& \\
\textbf{6d scene flow} & 6 & 48-126 & 1148-5352 & \textbf{ -2864.2 } & -2865.61 &  -2867.60 & -2877.08 \\ \hline
\end{tabular}
\caption{Dataset statistics and averaged maximum lower bound.
\# I denotes number of instances in dataset, $\abs{V}$ the number of nodes and $\abs{E}$ the number of edges in the underlying graphical model.
$\dagger$ means method is not applicable.
 \textbf{Bold numbers} indicate highest lower bound among competing algorithms.}  
\label{table:dataset_lower_bound_table}
\end{table*}

\begin{figure}
\begin{minipage}{0.47\columnwidth}
\centering
\includegraphics[width=0.7\textwidth]{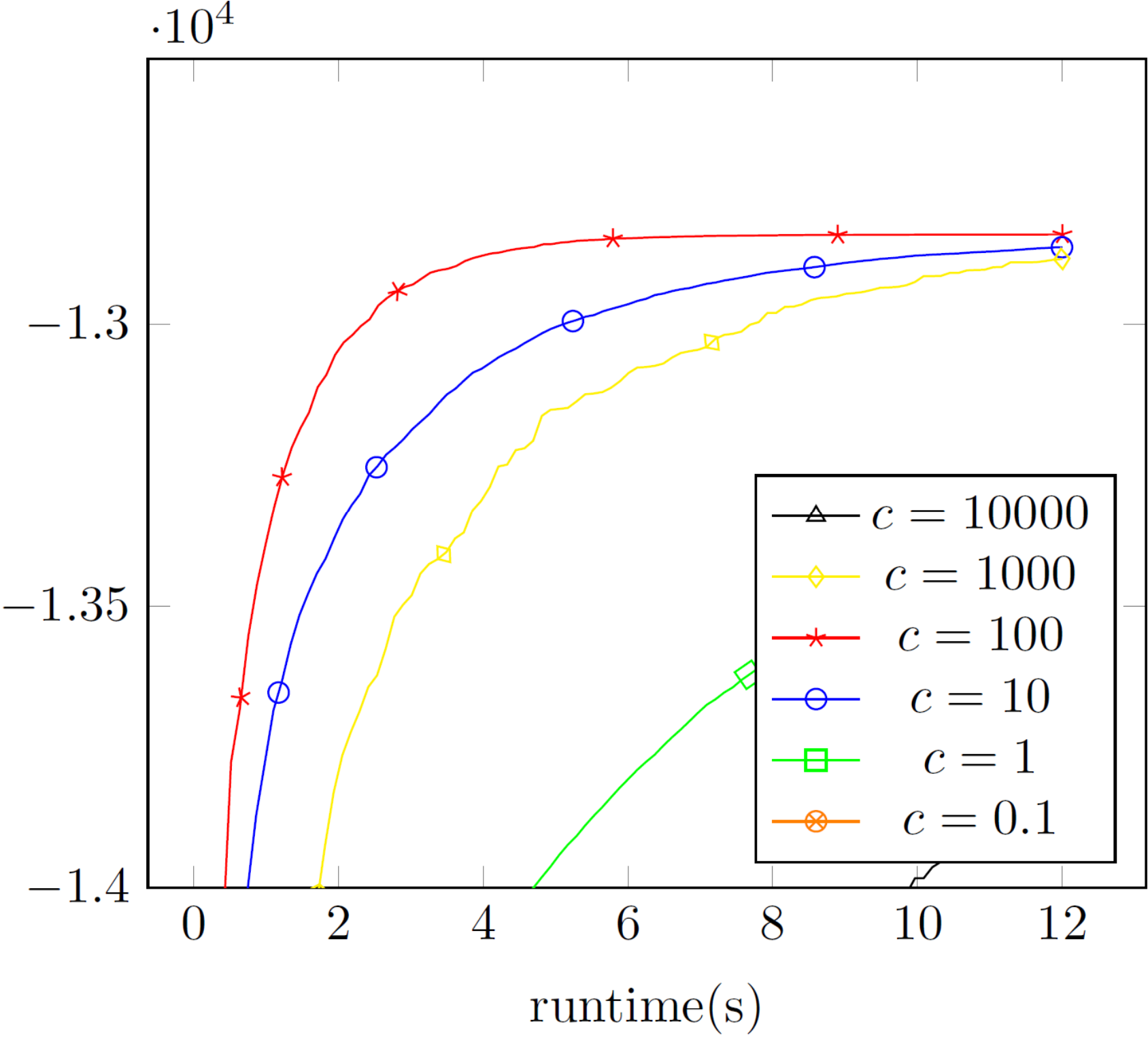} 
\caption{lower bound vs. runtime on instance 1CKK of the \textbf{protein folding} dataset solved with \texttt{FWMAP} and different values of proximal weight $c$ from~\eqref{eq:dual-decomposition-trust-region}.}
\label{fig:proximal_weight_comparison}
\end{minipage}
\hspace{10pt}
\begin{minipage}{0.47\columnwidth}
\centering
\includegraphics[width=0.7\textwidth]{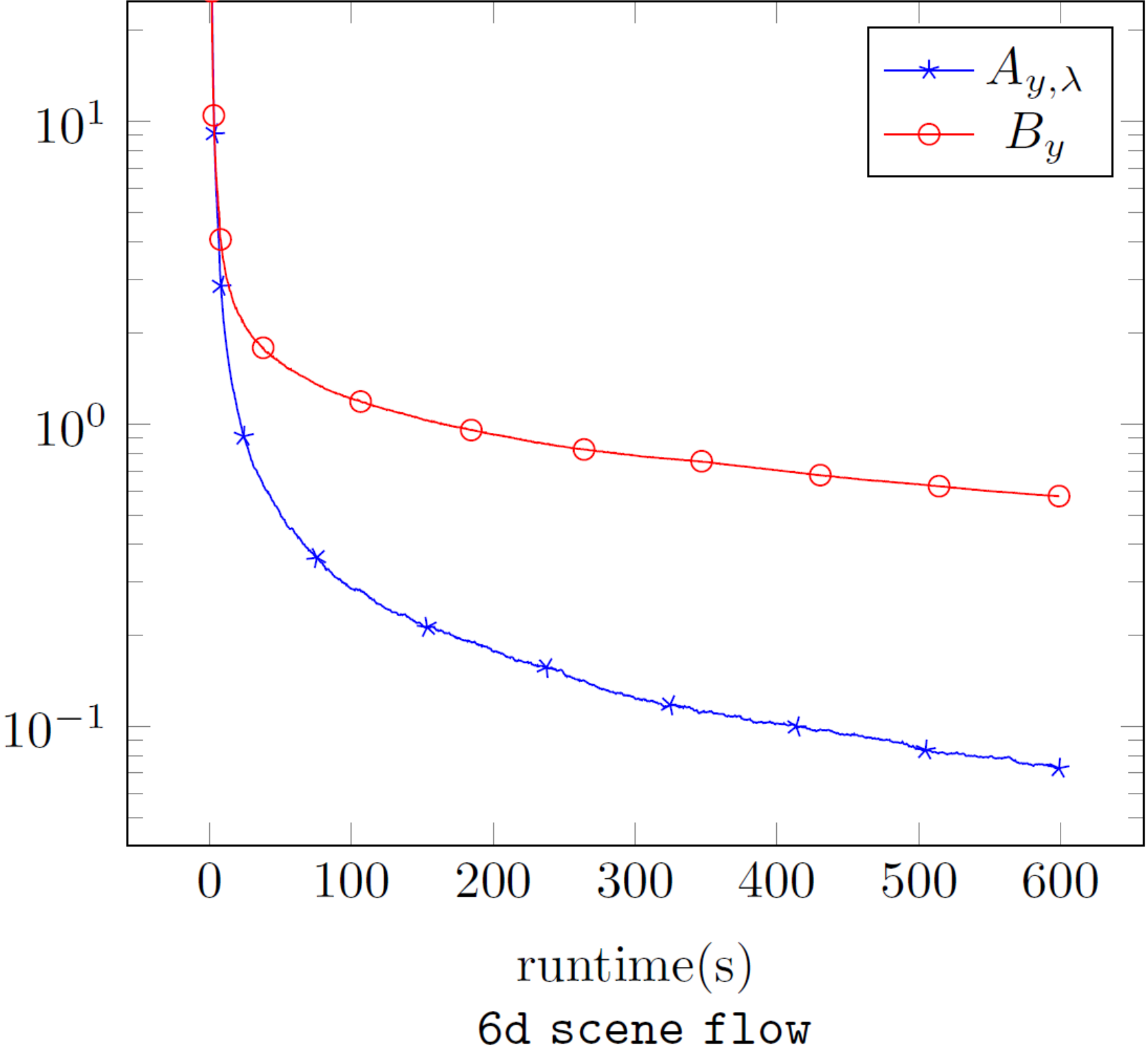} 
\caption{Duality gap quantities $A_{y,\lambda}$ and $B_{y}$ from~\eqref{eq:duality-gap-quantities} over time for the \textbf{synthetic 6 proj.} dataset from discrete tomography.}
\label{fig:duality-gap-quantities}
\end{minipage}
\end{figure}

We have chosen a selection of challenging MRF, discrete tomography and graph matching problems where message passing methods struggle or are not applicable.
    Detailed descriptions of these problems can be found in Appendix~\ref{sec:appendix}. 
\begin{remark}
Note that there are a large number of MRF and graph matching problems where message passing is the method of choice due to its greater speed and sufficient solution quality, see~\cite{OpenGMBenchmark,GraphMatchingMessagePassing}.
In such cases there is no advantage in using subgradient based solvers, which tend to be slower. 
However,  our chosen evaluation problems contain some of the most challenging MRF and graph matching problems with corresponding Lagrangean decompositions not solved satisfactorily with message passing solvers.
\end{remark}

We have excluded primal heuristics that do not solve a relaxation corresponding to our Lagrangean decomposition at all from comparison.
First, they do not deliver any lower bounds, hence cannot be directly compared.
Second, we see them as complementary solvers, since, when they are applied on the solved dual problem, their solution quality typically improves.
We have run every instance for 10 minutes on a Intel i5-5200U CPU.
Per-dataset plots showing averaged lower bound over time can be seen in Figure~\ref{fig:dataset_runtime_plots}.
Dataset statistics and final lower bounds averaged over datasets can be seen in Table~\ref{table:dataset_lower_bound_table}.
    Detailed results for each instance in each dataset can be found in in Appendix~\ref{sec:appendix}. 

\myparagraph{Solvers}
We focus our evaluation on methods that are able to handle general MAP inference problems
in which the access to subproblems is given by min-oracles. 
\begin{itemize}[leftmargin=1.5cm,noitemsep]
  \item[\texttt{FWMAP}:] Our solver as described in Section~\ref{sec:method}.
  \item[\texttt{CB}:] The state-of-the-art bundle method ConicBundle~\cite{ConicBundle}, which does not treat individual subproblems individually, but performs ascent on all Lagrangean multipliers $\lambda$ simultaneously.
  \item[\texttt{SA}:] Subgradient ascent with a Polyak step size rule.
  This solver was also used in~\cite{kappes2012bundle} for MRFs.
\end{itemize}
Additionally, we tested state-of-the-art versions of message passing (\texttt{MP}) solvers, when applicable.
\texttt{MP} is a popular method for MAP-MRF problems, and has recently been applied to the graph matching problem~\cite{GraphMatchingMessagePassing}.

All solvers optimize the same underlying linear programming relaxation.
Additionally, all subgradient based solvers also use the same decomposition.
This ensures that we compare the relevant solver methodology, not differences between relaxations or decompositions.

\myparagraph{Choice of proximal weight $c$ from~\eqref{eq:dual-decomposition-trust-region}}
The performance of our algorithm depends on the choice of the proximal weight parameter $c$ from~\eqref{eq:dual-decomposition-trust-region}.
A too large value will make each proximal step take long, while a too small value will mean too many proximal steps until convergence.
This behaviour can be seen in Figure~\ref{fig:proximal_weight_comparison}, where we have plotted lower bound against time for an MRF problem and \texttt{FWMAP} with different choices of proximal weight $c$.
We see that there is an optimal value of $100$, with larger and smaller values having inferior performance. 
However, we can also observe that performance of \texttt{FWMAP} is good for values an order of magnitude larger or smaller, hence \texttt{FWMAP} is not too sensitive on $c$.
It is enough to choose roughly the right order of magnitude for this parameter.

We have observed that the more subproblems there are in a problem decomposition~\eqref{eq:decomposition}, the smaller the proximal weight $c$ should be.
Since more subproblems usually translate to more complex dependencies in the decomposition, a smaller value of $c$ will be beneficial, as it makes the resulting more complicated proximal steps better conditioned.
A good formula for $c$ will hence be decreasing for increasing numbers of subproblems $\abs{T}$.
We have taken three instances out of the 50 we evaluated on and roughly fitted a curve that takes suitable values of proximal weight for these three instances, resulting in 
\begin{equation}
c =  \frac{1500000}{(\abs{T}+22)^2}\,.
\end{equation}

\myparagraph{Duality gap}
We have plotted the duality gap quantities $A_{y,\lambda}$ and $B_{y}$ from~\eqref{eq:duality-gap-quantities} for the \textbf{synthetic 6 proj.} dataset from discrete tomography in Figure~\ref{fig:duality-gap-quantities}.

\vspace{-4.3pt}
\subsection{Markov Random Fields.}
Most MRFs can be efficiently solved with message passing techniques~\cite{OpenGMBenchmark}, e.g.\ with the TRWS algorithm~\cite{TRWSKolmogorov}, which we denote by \texttt{MP}.
However, there are a few hard problems where TRWS gets stuck in suboptimal points, an example being the \textbf{protein folding} dataset~\cite{yanover2007minimizing} from computational biology.

\vspace{-4.3pt}
\subsection{Discrete tomography.}
In~\cite{SSVM_DT_2017} a dual decomposition based solver was proposed for the multi-label discrete tomography problem.
The decomposition was optimized with ConicBundle~\cite{ConicBundle}.
For our decomposition, message passing solvers are unfortunately not applicable.
The main problem seems that due to the unary potentials being zero, min-marginals for all subproblems are also zero.
Hence any min-marginal based step used in message passing will result in no progress. 
In other words, the initially zero Lagrangean multipliers are a local fix-point for message passing algorithms.
Therefore, we only compare against \texttt{SA} and \texttt{CB}.

\myparagraph{Datasets}
We compare on the synthetically generated text images from~\cite{SSVM_DT_2017}, denoted by \textbf{synthetic}.
These are $32 \times 32$ images of random objects with $2$, $4$ and $6$ projections directions.
We also compare on the classic \textbf{sheep logan} image with resolution $64 \times 64$ and $256 \times 256$ and $2$, $4$ and $6$ projections.  

\vspace{-4.3pt}
\subsection{Graph matching.}
As shown in~\cite{TorresaniEtAlQAP,HungarianBP}, Lagrangean decomposition based solvers are superior to solvers based on primal heuristics also in terms of the quality of obtained primal solutions.
In particular,~\cite{GraphMatchingMessagePassing} has proposed a message passing algorithm that is typically the method of choice and is on par/outperforms other message passing and subgradient based techniques on most problems.
We denote it by \texttt{MP}.

\myparagraph{Datasets}
There are a few problems where message passing based solvers proposed so far get stuck in suboptimal fixed points.
This behaviour occurred e.g.\ on the dataset~\cite{GraphFlow} in~\cite{GraphMatchingMessagePassing}, which we denote by \textbf{graph flow}.

\vspace{-4.3pt}
\subsection{Discussion}
Our solver \texttt{FWMAP} achieved the highest lower bound on each instance.
It was substantially better on the hardest and largest problems, e.g. on \textbf{sheep logan}.
While message passing solvers were faster (whenever applicable) in the beginning stages of the optimization, our solver \texttt{FWMAP} was fastest among subgradient based ones and eventually achieved a higher lower bound than the message passing one.
We also would like to mention that our solver had a much lower memory usage than the competing bundle solver \texttt{CB}.
On the larger problem instances, \texttt{CB} would often use all available memory on our 8~GB machine.

{\small
\bibliographystyle{ieee}
\bibliography{FW-MAP}

\begin{thebibliography}{10}\itemsep=-1pt

\bibitem{denseCRF-CVPR17}
T. Ajanthan, A. Desmaison, R. Bunel, M. Salzmann, P.H.S. Torr, and M.~P. Kumar.
\newblock Efficient linear programming for dense {CRF}s.
\newblock In {\em IEEE Conference on Computer Vision and Pattern Recognition
  (CVPR)}, 2017.

\bibitem{GraphFlow}
H.~A. Alhaija, A. Sellent, D. Kondermann, and C. Rother.
\newblock Graphflow - 6d large displacement scene flow via graph matching.
\newblock In {\em GCPR}, 2015.

\bibitem{FastMaximumConvolutionBussieck}
M. Bussieck, H. Hassler, G.~J. Woeginger, and U.~T. Zimmermann.
\newblock Fast algorithms for the maximum convolution problem.
\newblock {\em Oper. Res. Let.}, 15(3):133--141, 1994.

\bibitem{FWolfe56}
M. Frank and P. Wolfe.
\newblock An algorithm for quadratic programming.
\newblock {\em Naval Research Logistics Quarterly}, 3:149--154, 1956.

\bibitem{ForestsFramesAndGamesGabw1992}
H.~N. Gabow and H.~H. Westermann.
\newblock Forests, frames, and games: Algorithms for matroid sums and
  applications.
\newblock {\em Algorithmica}, 7(1):465, Jun 1992.

\bibitem{MPLP}
A. Globerson and T.~S. Jaakkola.
\newblock Fixing max-product: Convergent message passing algorithms for {MAP}
  {LP}-relaxations.
\newblock In {\em Conference on Neural Information Processing Systems (NIPS)},
  2007.

\bibitem{ConicBundle}
C. Helmberg.
\newblock The conicbundle library for convex optimization v0.3.11., 2011.
\newblock \url{https://www-user.tu-chemnitz.de/~helmberg/ConicBundle/}.

\bibitem{jaggi2013revisiting}
M. Jaggi.
\newblock Revisiting {Frank-Wolfe}: Projection-free sparse convex optimization.
\newblock In {\em International Conference on Machine Learning (ICML)}, pages
  427--435, 2013.

\bibitem{LagrangeanRelaxationJohnsonMalioutov}
J. Johnson, D.~M. Malioutov, and A.~S. Willsky.
\newblock Lagrangian relaxation for map estimation in graphical models.
\newblock In {\em 45th Annual Allerton Conference on Communication, Control and
  Computing}, 2007.

\bibitem{AcceleratedMAPJojic}
V. Jojic, S. Gould, and D. Koller.
\newblock Accelerated dual decomposition for {MAP} inference.
\newblock In {\em International Conference on Machine Learning (ICML)}, 2010.

\bibitem{OpenGMBenchmark}
J.~H. Kappes, B. Andres, F.~A. Hamprecht, C. Schn\"orr, S. Nowozin, D. Batra,
  S. Kim, B.~X. Kausler, T. Kr\"oger, J. Lellmann, N. Komodakis, B.
  Savchynskyy, and C. Rother.
\newblock A comparative study of modern inference techniques for structured
  discrete energy minimization problems.
\newblock {\em International Journal of Computer Vision}, 115(2):155--184,
  2015.

\bibitem{kappes2012bundle}
J.~H. Kappes, B. Savchynskyy, and C. Schn{\"o}rr.
\newblock A bundle approach to efficient {MAP}-inference by {L}agrangian
  relaxation.
\newblock In {\em Computer Vision and Pattern Recognition (CVPR), 2012 IEEE
  Conference on}, pages 1688--1695. IEEE, 2012.

\bibitem{Kiwiel1985Bundle}
K.~C. Kiwiel.
\newblock {\em Methods of Descent for Nondifferentiable Optimization}.
\newblock Lecture Notes in Computer Science. Springer-Verlag, 1985.

\bibitem{ProximityControlBundleMethodsKiwiel}
K.~C. Kiwiel.
\newblock Proximity control in bundle methods for convex nondifferentiable
  minimization.
\newblock {\em Mathematical Programming}, 46(1):105--122, Jan 1990.

\bibitem{TRWSKolmogorov}
V. Kolmogorov.
\newblock Convergent tree-reweighted message passing for energy minimization.
\newblock {\em IEEE Trans. Pattern Anal. Mach. Intell.}, 28(10):1568--1583,
  2006.

\bibitem{SRMPKolmogorov}
V. Kolmogorov.
\newblock A new look at reweighted message passing.
\newblock {\em IEEE Trans. Pattern Anal. Mach. Intell.}, 37(5):919--930, 2015.

\bibitem{kolmogorov15:power}
V. Kolmogorov, J. Thapper, and S. Živný.
\newblock The power of linear programming for general-valued csps.
\newblock {\em SIAM Journal on Computing}, 44(1):1--36, 2015.

\bibitem{DualDecompositionKomodakis}
N. Komodakis, N. Paragios, and G. Tziritas.
\newblock {MRF} energy minimization and beyond via dual decomposition.
\newblock {\em IEEE Transactions on Pattern Analysis and Machine Intelligence},
  33(3):531--552, March 2011.

\bibitem{SSVM_DT_2017}
J. Kuske, P. Swoboda, and S. Petra.
\newblock A novel convex relaxation for non-binary discrete tomography.
\newblock In {\em Scale Space and Variational Methods in Computer Vision - 6th
  International Conference, {SSVM} 2017, 2017, Proceedings}, pages 235--246,
  2017.

\bibitem{BCFW}
S. Lacoste-Julien, M. Jaggi, M. Schmidt, and P. Pletscher.
\newblock Block-coordinate {F}rank-{W}olfe optimization for structural {SVM}s.
\newblock In {\em International Conference on Machine Learning (ICML)}, 2013.

\bibitem{Lemarechal1986Bundle}
C. Lemarechal.
\newblock Constructing bundle methods for convex optimization.
\newblock {\em North-Holland Mathematics Studies}, 129:201--240, 1986.

\bibitem{MapMirrorDescent}
D.~V.~N. Luong, P. Parpas, D. Rueckert, and B. Rustem.
\newblock Solving {MRF} minimization by mirror descent.
\newblock In {\em Advances in Visual Computing - 8th International Symposium,
  {ISVC} 2012, Rethymnon, Crete, Greece, July 16-18, 2012, Revised Selected
  Papers, Part {I}}, pages 587--598, 2012.

\bibitem{Meshi:NIPS15}
Ofer Meshi, Mehrdad Mahdavi, and Alex Schwing.
\newblock Smooth and strong: {MAP} inference with linear convergence.
\newblock In {\em Conference on Neural Information Processing Systems (NIPS)},
  2015.

\bibitem{Osokin:ICML16}
A. Osokin, J.-B. Alayrac, I. Lukasewitz, P.~K. Dokania, and S. Lacoste-Julien.
\newblock Minding the gaps for block {F}rank-{W}olfe optimization of structured
  {SVM}s.
\newblock In {\em International Conference on Machine Learning (ICML)}, 2016.

\bibitem{OsokinSubmodularRelaxationMRF}
A. Osokin and D.~P. Vetrov.
\newblock Submodular relaxation for inference in markov random fields.
\newblock {\em IEEE Transactions on Pattern Analysis \& Machine Intelligence},
  37(7):1347--1359, July 2015.

\bibitem{RavikumarProximalMethodsMAPMRF}
P. Ravikumar, A. Agarwal, and M.~J. Wainwright.
\newblock Message-passing for graph-structured linear programs: Proximal
  methods and rounding schemes.
\newblock {\em Journal of Machine Learning Research}, 11:1043--1080, 2010.

\bibitem{rockafellar1976monotone}
R~Tyrrell Rockafellar.
\newblock Monotone operators and the proximal point algorithm.
\newblock {\em SIAM journal on control and optimization}, 14(5):877--898, 1976.

\bibitem{savchynskyy2011study}
B. Savchynskyy, J. Kappes, S. Schmidt, and C. Schn{\"o}rr.
\newblock A study of nesterov's scheme for lagrangian decomposition and map
  labeling.
\newblock In {\em Computer Vision and Pattern Recognition (CVPR), 2011 IEEE
  Conference on}, pages 1817--1823. IEEE, 2011.

\bibitem{AdaptiveDiminishingSmoothingSavchynskyyUAI}
B. Savchynskyy, S. Schmidt, J.~H. Kappes, and C. Schn\"{o}rr.
\newblock Efficient {MRF} energy minimization via adaptive diminishing
  smoothing.
\newblock In {\em Uncertainty in Artificial Intelligence (UAI)}, 2012.

\bibitem{SchlesingerSubgradient}
M.~I. Schlesinger and V.~V. Giginyak.
\newblock Solution to structural recognition ({MAX},+)-problems by their
  equivalent transformations.
\newblock (2):3--18, 2007.

\bibitem{schmidt2011evaluationProximalMAP}
S. Schmidt, B. Savchynskyy, J.~H. Kappes, and C. Schn{\"o}rr.
\newblock Evaluation of a first-order primal-dual algorithm for mrf energy
  minimization.
\newblock In {\em International Workshop on Energy Minimization Methods in
  Computer Vision and Pattern Recognition}, pages 89--103. Springer, 2011.

\bibitem{SchraudolphReweightedPerfectMatching}
N.~N. Schraudolph.
\newblock Polynomial-time exact inference in {NP}-hard binary {MRF}s via
  reweighted perfect matching.
\newblock In {\em 13$^{th}$ Intl.\ Conf.\ Artificial Intelligence and
  Statistics (AIstats)}, volume~9 of {\em Workshop and Conference Proceedings},
  pages 717--724. Journal of Machine Learning Research (JMLR), 2010.

\bibitem{Schwing:ICML14}
Alexander Schwing, Tamir Hazan, Marc Pollefeys, and Raquel Urtasun.
\newblock Globally convergent parallel {MAP} {LP} relaxation solver using the
  {F}rank-{W}olfe algorithm.
\newblock In {\em International Conference on Machine Learning (ICML)}, 2014.

\bibitem{MP-BCFW}
N. Shah, V. Kolmogorov, and C.~H. Lampert.
\newblock A multi-plane block-coordinate {F}rank-{W}olfe algorithm for training
  structural {SVM}s with a costly max-oracle.
\newblock In {\em IEEE Conference on Computer Vision and Pattern Recognition
  (CVPR)}, 2015.

\bibitem{SonGloJaa_optbook}
David Sontag, Amir Globerson, and Tommi Jaakkola.
\newblock Introduction to dual decomposition for inference.
\newblock In Suvrit Sra, Sebastian Nowozin, and Stephen~J. Wright, editors,
  {\em Optimization for Machine Learning}, pages 219--254. MIT Press, 2012.

\bibitem{StorvikDahlLagrangeanBasedMAP}
G. Storvik and G. Dahl.
\newblock Lagrangian-based methods for finding map.
\newblock {\em IEEE Trans. on Image Processing}, 9(3):469--479, march 2000.

\bibitem{IRPSLP}
P. Swoboda, J. Kuske, and B. Savchynskyy.
\newblock A dual ascent framework for {L}agrangean decomposition of
  combinatorial problems.
\newblock In {\em CVPR}, 2017.

\bibitem{GraphMatchingMessagePassing}
P. Swoboda, C. Rother, H. Abu~Alhaija, D. Kainmueller, and B. Savchynskyy.
\newblock Study of {L}agrangean decomposition and dual ascent solvers for graph
  matching.
\newblock In {\em CVPR}, 2017.

\bibitem{TorresaniEtAlQAP}
L. Torresani, V. Kolmogorov, and C. Rother.
\newblock A dual decomposition approach to feature correspondence.
\newblock {\em IEEE Trans. Pattern Anal. Mach. Intell.}, 35(2):259--271, 2013.

\bibitem{WainwrightJordan08}
Martin~J Wainwright, Michael~I Jordan, et~al.
\newblock Graphical models, exponential families, and variational inference.
\newblock {\em Foundations and Trends{\textregistered} in Machine Learning},
  1(1--2):1--305, 2008.

\bibitem{Werner07}
T. Werner.
\newblock A linear programming approach to max-sum problem: {A} review.
\newblock {\em IEEE Trans. Pattern Analysis and Machine Intelligence},
  29(7):1165--1179, 2007.

\bibitem{yanover2007minimizing}
Chen Yanover, Ora Schueler-Furman, and Yair Weiss.
\newblock Minimizing and learning energy functions for side-chain prediction.
\newblock In {\em Annual International Conference on Research in Computational
  Molecular Biology}, pages 381--395. Springer, 2007.

\bibitem{yarkony2010covering}
Julian Yarkony, Charless Fowlkes, and Alexander Ihler.
\newblock Covering trees and lower-bounds on quadratic assignment.
\newblock In {\em Computer Vision and Pattern Recognition (CVPR), 2010 IEEE
  Conference on}, pages 887--894. IEEE, 2010.

\bibitem{HungarianBP}
Z. Zhang, Q. Shi, J. McAuley, W. Wei, Y. Zhang, and A. van~den Hengel.
\newblock Pairwise matching through max-weight bipartite belief propagation.
\newblock In {\em The IEEE Conference on Computer Vision and Pattern
  Recognition (CVPR)}, June 2016.

\end{thebibliography}
}

\clearpage
\appendix
\section{Proofs}
\label{sec:appendix}

The BLP relaxation~\cite{kolmogorov15:power} introduces
a probability distribution $\mu_i$ over $\{0,1\}$ for each $i\in[d]$
and a probability distribution $\mu_t$ over $\dom f_t$ for each $t\in T$.
It can be written as follows:
\begin{equation}\label{eq:BLP}
\begin{array}{rll}
\min\limits_{\mu\ge {\bf 0}} & \sum\limits_{t\in T}\sum\limits_{z\in \dom f_t} \mu_t(z)f_t(z) 
\\
  \mbox{s.t.} & \mu_i(0)+\mu_i(1)=1 & \forall i\in [d] \\
              & \sum\limits_{z\in\dom f_t} \mu_t(z)=1 & \forall t\in T \\
              & \!\!\!\!\!\!\sum\limits_{z\in\dom f_t:z_i=a} \!\!\!\!\!\mu_t(z)_i=\mu_{i}(a) & \forall t\in T,i\in A_t,a\in\{0,1\} 
\end{array} 
\end{equation}
Let us show that the optimal values of~\eqref{eq:BLP} and~\eqref{eq:decomposition1:relaxation} coincide.

\begin{proof}[Proof of equivalence of~\eqref{eq:BLP} and~\eqref{eq:decomposition1:relaxation}]
Define an extension $\hat f_t:\mathbb R^{A_t}\rightarrow\mathbb R\cup\{+\infty\}$
of function $\hat f:\{0,1\}^{A_t}\rightarrow\mathbb R\cup\{+\infty\}$ as follows: for 
a vector $x\in\mathbb R^{A_t}$ set
\begin{equation}\label{eq:BLP:t}
\begin{array}{rll}
\hat f_t(x)=\min\limits_{\mu_t\ge {\bf 0}} & \sum\limits_{z\in \dom f_t} \mu_t(z)f_t(z) 
\\
  \mbox{s.t.} & \sum\limits_{z\in\dom f_t} \mu_t(z)=1 &  \\
              & \!\!\!\!\!\!\sum\limits_{z\in\dom f_t} \!\!\!\!\!\mu_t(z)\cdot z=x
\end{array}
\end{equation}
Note, if $x\notin[0,1]^{A_t}$ then~\eqref{eq:BLP:t} does not have a feasible solution,
and so $\hat f_t(x)=+\infty$. Observe that the constraints in the last line of~\eqref{eq:BLP}
for $a=0$ are redundant - they follow from the remaining constraints.
Also observe that constraints $\sum_{z\in\dom f_t:z_i=1} \mu_t(z)_i=\mu_{i}(1)$
for $i\in A_t$ can be written as $\sum_{z\in\dom f_t} \mu_t(z)\cdot z=x$
if we denote $x_i=\mu_{i}(1)$ for $i\in A_t$.
Therefore, problem~\eqref{eq:BLP} can be equivalently rewritten as follows:
\begin{equation}\label{eq:BLP2}
\min_{x\in\mathbb R^d} \; \sum_{t\in T}\hat f(x_{A_t})
\end{equation}
It can be seen that the last problem is equivalent to~\eqref{eq:decomposition1:relaxation}.
Indeed, we just need to observe that for each $t\in T$ and $x\in \mathbb R^{A_t}$ we have
$$
\min_{\substack{y\in\conv(\SY_t) \\ y_\ast=x}} y_\circ
=
\min_{\substack{\alpha\ge{\bf 0},\;\sum_{z\in\dom f_t} \alpha(z)=1 \\ y=\sum_{z\in\dom f_t} \alpha(z)\cdot [z\; f(z)] \\ y_\ast=x}} y_\circ \hspace{30pt} 
$$
$$
\begin{array}{c}
=\min\limits_{\substack{\alpha\ge{\bf 0},\;\sum_{z\in\dom f_t} \alpha(z)=1 \\ \sum_{z\in\dom f_t} \alpha(z)\cdot z=x}} \sum_{z\in\dom f_t}\alpha(z)f(z)  = \hat f_t(x)
\end{array}
$$
\end{proof}

\begin{proof}[Proof of Proposition~\ref{prop:dual-of-decomposition}]
Write $f(y) := \sum_{t\in T}y^t_\circ$, then problem~\eqref{eq:decomposition1:relaxation} can be written as
\begin{equation}
  \min_{\substack
{ (y,x) \in \Y\times\mathbb R^d
\\y^t_\star = x_{A_t}\; \forall t\in T}
} \;\;
f(y)
\end{equation}
The Lagrangian w.r.t.\ the equality constraints is given by
\begin{eqnarray*}
L(y,x,\lambda)
&=&f(y)+\sum_{t\in T}\langle y^t_\star-x_{A_t},\lambda^t\rangle \\
&=& \sum_{t\in T}\langle y^t,[\lambda^t\; 1] \rangle - \sum_{t\in T} \langle x_{A_t},\lambda^t \rangle
\end{eqnarray*}
Therefore, the dual function for $\lambda\in\bigotimes_{t\in T}\mathbb R^{A_t}$ is
\begin{eqnarray*}
h(\lambda)&=&\min_{(y,x) \in \Y\times\mathbb R^d}L(y,x,\lambda) \\
&=&\begin{cases}
\sum\limits_{t\in T}\min\limits_{y^t\in\Y_t}\langle y^t,[\lambda^t\; 1]\rangle & \mbox{if }\lambda\in\Lambda \\
-\infty      & \mbox{otherwise}
\end{cases}
\end{eqnarray*}
The problem  can thus be formulated as $\max_\lambda h(\lambda)$,
or equivalently as $\max_{\lambda\in\Lambda} h(\lambda)$. This coincides with
formulation given in Proposition~\ref{prop:dual-of-decomposition}.

Since constraint $y\in\Y$ can be expressed as a linear program,
the duality between~\eqref{eq:decomposition1:relaxation}
and~\eqref{eq:dual-decomposition} can be viewed as a special case of linear programming (LP)
duality (where the value of function $h(\lambda)$ is also written
as a resulting of some LP). For LPs it is known that strong duality holds assuming
that either the primal or the dual problems have a feasible solution.
This holds in our case, since vector $\lambda={\bf 0}\in\Lambda$ is feasible.
We can conclude that we have a strong duality between \eqref{eq:decomposition1:relaxation}
and~\eqref{eq:dual-decomposition}.
\end{proof}

\begin{proof}[Proof of Proposition~\ref{prop:dual-of-decomposition-plus-trust-region}]
First, we derive the dual of $h_{\mu,c}$:
\begin{equation*}
\begin{array}{rl}
 & 
\max\limits_{\lambda \in \Lambda} h_{\mu,c}(\lambda)
\\
= &
\max\limits_{\lambda \in \Lambda} \sum\limits_{t\in T} \min\limits_{y^t \in \Y_t} \langle y^t, [\lambda^t\ 1] \rangle - \frac{1}{2c} \norm{\lambda^t - \mu^t}^2
\\
= &
\min\limits_{y \in \Y} 
\underbrace{
\max\limits_{\lambda \in \Lambda} 
\sum\limits_{t \in T} \langle y^t, [\lambda^t\ 1] \rangle - \frac{1}{2c} \norm{\lambda^t - \mu^t}^2
}_{=: f_{\mu,c}(y)}
\\
\end{array} 
\end{equation*}
The function $f_{\mu,c}(y)$ has a closed form expression, since it is a quadratic function subject to linear equalities.
Write $\nu_i = \frac{1}{|T_i|} \sum_{t\in T_i} (c \cdot y^t_i + \mu^t_i)$ for $i \in [d]$.
The $\argmax$ in the expression defining $f_{\mu,c}(y)$ are
\begin{equation}
\lambda^t = ( c \cdot y^t_\star + \mu^t ) - \nu_{A_t}
\end{equation}

The function value is
\begin{equation*}
\begin{array}{rl}
f_{\mu,c}(x) 
&
= \sum\limits_{t \in T} \langle y^t , [ \lambda^t_\star\ 1] \rangle - \frac{1}{2c} \norm{\lambda^t_\star - \mu^t}^2
\\
&
=
\sum\limits_{t \in T} \left( \begin{array}{cc}\langle y^t_\star, c\cdot y^t_\star + \mu^t - \nu_{A_t} \rangle + y^t_\circ \\- \frac{1}{2c} \norm{c x^t_\star  + \mu^t - \nu_{A_t} - \mu^t}^2 \end{array} \right)
\\
&
=
\sum\limits_{t \in T} \left( \begin{array}{cc}c \norm{y^t_\star}^2 + \langle y^t_\star, \mu^t - \nu_{A_t} \rangle + y^t_\circ  \\ - \frac{1}{2c} \norm{c y^t_\star - \nu_{A_t}}^2 \end{array} \right)
\\
&
=
\sum\limits_{t \in T} \left( \begin{array}{cc} c \norm{y^t_\star}^2 + \langle y^t_\star, \mu^t - \nu_{A_t} \rangle + y^t_\circ \\ - \frac{1}{2c} \left\{\norm{c y^t_\star}^2 - 2c \langle y^t_\star, \nu_{A_t}\rangle + \norm{\nu_{A_t}}^2 \right\} \end{array} \right)
\\
&
=
\sum\limits_{t \in T} \left( \frac{c}{2} \norm{y^t_\star}^2 + \langle y^t_\star, \mu^t \rangle + y^t_\circ - \frac{1}{2c} \norm{\nu_{A_t}}^2 \right)\,.
\end{array}
\end{equation*}
The gradient is
$\nabla_t f_{\mu,c}(y) = [ c\cdot y^t + \mu^t - \nu_{A_t} \ 1] = [\lambda^t\ 1]$.
\end{proof}

\begin{proof}[Proof of Proposition~\ref{prop:AB}] 
Let $\overline{\Y\times\mathbb R^d}$ be the set of vectors $(y,x)\in \Y\times\mathbb R^d$
satisfying the equality constraints $y^t_\star=x_{A_t}$ for all $t$.
By construction, for any  $\lambda\in\Lambda$ we have 
\begin{subequations}
\begin{eqnarray}
f(y)&\!\!\!\!=\!\!\!\!& L(y,x,\lambda) \quad \forall (y,x)\in \overline{\Y\times\mathbb R^d}\quad \label{eq:AB:proof:a} \\
\hspace{-10pt}L(y,x,\lambda)&\!\!\!\!\ge\!\!\!\!&  h(\lambda) \hspace{31pt} \forall (y,x)\in {\Y\times\mathbb R^d} \label{eq:AB:proof:b} \\
\hspace{-10pt}L(y,x,\lambda)&\!\!\!\!=\!\!\!\!&\sum_{t\in T}\langle y^t,[\lambda^t\; 1] \rangle \label{eq:AB:proof:c}
\end{eqnarray}
\end{subequations}
Eq.~\eqref{eq:AB:proof:c} gives that $A_{y,\lambda}=L(y,x,\lambda)-h(y)$
for any  $(y,\lambda)\in\Y\times\Lambda$ and $x\in\mathbb R^d$, and so from~\eqref{eq:AB:proof:b} we get that $A_{y,\lambda}\ge 0$.
Clearly, we have $B_y\ge 0$. The following two facts imply part (b) of Proposition~\ref{prop:AB}:
\begin{itemize}
\item Consider vector $y\in\Y$. Then $B_y=0$ if and only if $(y,x)\in\overline{\Y\times\mathbb R^d}$ for some $x$.
(This can be seen from the definition of $B_y$ in Section~\ref{sec:duality-gap}).
\item Consider vectors $(y,x)\in \overline{\Y\times\mathbb R^d}$ and $\lambda\in\Lambda$.
They are an optimal primal-dual pair if and only if $f(y)=h(\lambda)$,
which in turn holds if and only if $A_{y,\lambda}=0$ (since $A_{y,\lambda}=L(y,x,\lambda)-h(\lambda)=f(y)-h(\lambda)$).
\end{itemize}
It remains to show inequality~\eqref{eq:gap-bound}.
Denote $\delta=\lambda^\ast-\lambda$, then $\sum\limits_{t\in T_i}\delta^t_i=0$ for any $i\in[d]$.
Denoting $y^-_i=\min\limits_{t\in T_i}y^t_i$ and $y^+_i=\max\limits_{t\in T_i}y^t_i$, we get
\begin{eqnarray*}
\sum_{t\in T_t} y^t_i\cdot \delta^t_i
&=&\sum_{t\in T_t} \left[y^t_i-y^-_i\right]\cdot \delta^t_i \\
&\le &\sum_{t\in T_t} \left[y^+_i-y^-_i\right]\cdot |\delta^t_i| \\
&\le & \left[y^+_i-y^-_i\right]\cdot \norm{\delta}_{1,\infty}
\end{eqnarray*}
Summing these inequalities over $i\in[d]$ gives
$$
\sum_{t\in T}\langle y^t_\star,\delta^t \rangle
\le B_y\cdot \norm{\delta}_{1,\infty}
$$
Recalling that $A_{\lambda^\ast,y}\ge 0$, we obtain the desired claim:
\begin{eqnarray*}
h(\lambda^\ast)
&\le& \sum_{t\in T}\langle y^t,[(\lambda^\ast)^t\; 1]\rangle \\
&=& \sum_{t\in T}\langle y^t,[\lambda^t\; 1]\rangle + \sum_{t\in T} \langle y^t_\star,\delta^t\rangle \\
&\le & \sum_{t\in T}\langle y^t,[\lambda^t\; 1]\rangle + B_y\cdot \norm{\delta}_{1,\infty}
\end{eqnarray*}
\end{proof}

\begin{lemma}[step size in Algorithm~\ref{alg:BCFW}]
\label{lemma:step-size-BCFW}
The optimal step size $\gamma$ in Algorithm~\ref{alg:BCFW} is
\begin{equation}
\gamma = 
\frac{\langle \nabla_t f_{\mu,c}(y) , y^t - z^t \rangle}
{c \norm{y^t_\star - z^t_\star}^2}
=
\frac{\langle [c\cdot y^t_\star + \mu^t - \nu_{A_t}\ 1] , y^t - z^t \rangle}
{c \norm{y^t_\star - z^t_\star}^2}
\end{equation}
and clip $\gamma$ to [0,1].
\end{lemma}
\begin{proof}
Recall that $y(\gamma)$ in algorithm~\ref{alg:BCFW} is defined as
$y(\gamma)^s = \left\{ \begin{array}{ll} y^s, & s \neq t \\ (1-\gamma)y^t + \gamma z^t,& s = t \end{array} \right.$.
The derivative $f_{\mu,c}(y(\gamma))' = \langle  \nabla f_{\mu,c}(y(\gamma)), y(\gamma)' \rangle$ is hence zero except in the $t$-th place.
Thus,
\begin{equation}
\begin{array}{rl}
&
f_{\mu,c}(x(\gamma))' \\
= & 
 \langle \nabla_t f_{\mu,c}(y), - y^t + z^t \rangle
\\
= &
\langle [ c \cdot y^t_\star(\gamma) + \mu^t - \nu_{A_t}\ 1], -y^t + z^t \rangle
\\
= &
\begin{array}{c}\langle [ c \cdot y^t_\star + \mu^t - \nu_{A_t}\ 1], -y^t + z^t \rangle \\ + \gamma \langle c\cdot(-y^t_\star + z^t_\star), -y^t_\star + z^t_\star \rangle \end{array}
\end{array}
\end{equation}
Setting the above derivative zero yields
$$\gamma = 
\frac{\langle [c\cdot y^t_\star + \mu^t - \nu_{A_t}\ 1] , y^t - z^t \rangle}
{c \norm{y^t_\star - z^t_\star}^2} \,.
$$
Recalling that we require $\gamma \in [0,1]$, we get the desired formula.  
\end{proof}

\section{Detailed experimental evaluation}

In Table~\ref{fig:instance-table} we give the final lower bound obtained by each tested algorithm for every instance of every dataset we evaluated on.
The averaged numbers are given in Table~\ref{table:dataset_lower_bound_table}.

{\onecolumn
\centering
\begin{longtable}{|lcccc|}
\caption{Lower bound of each instance. $\dagger$ means method not applicable. \textbf{Bold} numbers indicate highest lower bound among competing methods.} \\ \hline 
\label{fig:instance-table}
Instance & \texttt{FWMAP} & \texttt{CB} & \texttt{SA} & \texttt{MP} \\ \hline \endhead
\multicolumn{ 5 }{| c |}{ \underline{MRF}} \\ 
\multicolumn{ 5 }{| c |}{ \textbf{ protein folding } } \\ \hline 
1CKK & \textbf{ -12840.23 } & -12857.29 &  -12945.39 &-12924.97  \\ \hline
1CM1 & \textbf{ -12486.15 } & -12499.21 &  -12591.23 &-12488.10  \\ \hline
1SY9 & \textbf{ -9193.38 } & -9196.14 &  -9293.58 &-9194.77  \\ \hline
2BBN & \textbf{ -12396.51 } & -12461.89 &  -12585.85 & -12417.20  \\ \hline
2BCX & \textbf{ -14043.57 } & -14144.89 &  -14231.86 &-14112.73  \\ \hline
2BE6 & \textbf{ -13311.78 } & -13381.35 &  -13410.24 &-13438.23  \\ \hline
2F3Y & \textbf{ -14572.71 } & -14619.70 &  -14672.71 &-14641.60  \\ \hline
2FOT & \textbf{ -12049.52 } & -12112.31 &  -12154.66 &-12103.75  \\ \hline
2HQW & \textbf{ -13514.79 } & -13573.99 &  -13610.14 &-13539.69  \\ \hline
2O60 & \textbf{ -13557.32 } & -13664.00 &  -13718.71 &-13565.42  \\ \hline
3BXL & \textbf{ -14125.86 } & -14165.97 &  -14266.01 &-14136.79  \\ \hline
\multicolumn{ 5 }{| c |}{ \underline{Discrete tomography}} \\ 
\multicolumn{ 5 }{| c |}{ \textbf{ 2 projections } } \\ \hline
0.10\_0.10\_2 & \textbf{ 97.99 } & 97.94 & 96.46 & $\dagger$ \\ \hline
0.20\_0.20\_2 & \textbf{ 226.81 } & 226.66 & 222.05 & $\dagger$ \\ \hline
0.30\_0.30\_2 & \textbf{ 205.65 } & 205.25 & 194.49 & $\dagger$ \\ \hline
0.40\_0.40\_2 & \textbf{ 271.23 } & 270.99 & 253.94 & $\dagger$ \\ \hline
0.50\_0.48\_87 & \textbf{ 340.13 } & 339.98 & 315.41 & $\dagger$ \\ \hline
0.60\_0.58\_28 & \textbf{ 313.19 } & 312.80 & 288.73 & $\dagger$ \\ \hline
0.70\_0.67\_47 & \textbf{ 287.11 } & 286.83 & 246.04 & $\dagger$ \\ \hline
0.80\_0.76\_72 & \textbf{ 338.97 } & 338.78 & 290.73 & $\dagger$ \\ \hline
0.90\_0.85\_63 & \textbf{ 313.98 } & 313.77 & 246.63 & $\dagger$ \\ \hline
\multicolumn{ 5 }{| c |}{ \textbf{ 4 projections } } \\ \hline
0.10\_0.10\_2 & \textbf{ 102.00 } & 101.55 & 99.50 & $\dagger$ \\ \hline
0.20\_0.20\_2 & \textbf{ 250.61 } & 250.02 & 245.30 & $\dagger$ \\ \hline
0.30\_0.30\_2 & \textbf{ 247.86 } & 246.44 & 233.65 & $\dagger$ \\ \hline
0.40\_0.40\_2 & \textbf{ 365.05 } & 364.00 & 346.89 & $\dagger$ \\ \hline
0.50\_0.48\_87 & \textbf{ 439.60 } & 435.50 & 412.32 & $\dagger$ \\ \hline
0.60\_0.58\_28 & \textbf{ 400.91 } & 400.05 & 368.05 & $\dagger$ \\ \hline
0.70\_0.67\_47 & \textbf{ 393.88 } & 392.57 & 371.80 & $\dagger$ \\ \hline
0.80\_0.76\_72 & \textbf{ 443.87 } & 440.91 & 413.42 & $\dagger$ \\ \hline
0.90\_0.85\_63 & \textbf{ 397.14 } & 395.93 & 358.60 & $\dagger$ \\ \hline
\multicolumn{ 5 }{| c |}{ \textbf{ 6 projections } } \\ \hline
0.10\_0.10\_2 & \textbf{ 102.00 } & \textbf{ 102.00 } & 101.82 & $\dagger$ \\ \hline
0.20\_0.20\_2 & \textbf{ 256.00 } & 255.85 & 254.74 & $\dagger$ \\ \hline
0.30\_0.30\_2 & \textbf{ 295.85 } & 292.28 & 272.74 & $\dagger$ \\ \hline
0.40\_0.40\_2 & \textbf{ 461.27 } & 456.70 & 433.89 & $\dagger$ \\ \hline
0.50\_0.48\_87 & \textbf{ 533.95 } & 526.86 & 494.29 & $\dagger$ \\ \hline
0.60\_0.58\_28 & \textbf{ 514.05 } & 507.34 & 474.61 & $\dagger$ \\ \hline
0.70\_0.67\_47 & \textbf{ 577.38 } & 566.15 & 530.47 & $\dagger$ \\ \hline
0.80\_0.76\_72 & \textbf{ 542.96 } & 534.01 & 488.62 & $\dagger$ \\ \hline
0.90\_0.85\_63 & \textbf{ 535.78 } & 518.67 & 468.60 & $\dagger$ \\ \hline
\multicolumn{ 5 }{| c |}{ \textbf{ sheep logan 64x64 } } \\ \hline
Logan\_64\_2 & \textbf{ 582.52 } & 541.62 & 392.47 & $\dagger$ \\ \hline
Logan\_64\_4 & \textbf{ 871.58 } & 831.63 & 702.32 & $\dagger$ \\ \hline
Logan\_64\_6 & \textbf{ 1237.44 } & 1170.36 & 1011.00 & $\dagger$ \\ \hline
\multicolumn{ 5 }{| c |}{ \textbf{ sheep logan 256x256 } } \\ \hline
Logan\_256\_2 & \textbf{ 3709.46 } & 3505.46 & 2599.41 & $\dagger$ \\ \hline
Logan\_256\_4 & \textbf{ 4888.25 } & 4739.40 & 976.29 & $\dagger$ \\ \hline
Logan\_256\_6 & \textbf{ 5142.48 } & 4832.85 & -2463.81 & $\dagger$ \\ \hline
\multicolumn{ 5 }{| c |}{ \underline{Graph matching}} \\*
\multicolumn{ 5 }{| c |}{ \textbf{ 6d scene flow } } \\* \hline
board & \textbf{ -2262.66 } & -2262.66 &  -2262.89 &  \textbf{ -2262.66 } \\ \hline
books & \textbf{ -4179.79 } & -4186.16 &  -4191.30 &  -4204.14 \\ \hline
hammer & \textbf{ -2125.87 } & -2127.66 &  -2130.58 &  -2146.81 \\ \hline
party & \textbf{ -3648.03 } & -3648.71 &  -3649.41 &  -3657.12 \\ \hline
table & \textbf{ -3340.59 } & -3341.12 &  -3343.81 &  -3363.98 \\ \hline
walking & \textbf{ -1627.30 } & -1627.34 &  -1627.58 &  -1627.79 \\ \hline
\end{longtable}
}

\end{document}